\documentclass[final,12pt]{clear2023} % Include author names

\title[Causal Reductions]{Combining Interventional and Observational Data Using Causal Reductions}

\usepackage{times}
\usepackage{hyperref}       % hyperlinks
\usepackage{url}            % simple URL typesetting
\usepackage{booktabs}       % professional-quality tables
\usepackage{amsfonts}       % blackboard math symbols
\usepackage{nicefrac}       % compact symbols for 1/2, etc.
\usepackage{microtype}      % microtypography
\usepackage{xcolor}         % colors
\usepackage{stackengine}
\usepackage{multirow}

\usepackage{tikz} % nice language for creating drawings and diagrams

\usepackage{proof-at-the-end}
\usepackage{caption}
\usepackage{float}
\usepackage{comment}
\usepackage{makecell}
\usepackage{algorithmic,algorithm}
\newtheorem{sa}{Theorem}[section]
\newtheorem{Thm}[sa]{Theorem}
\newtheorem{Cor}[sa]{Corollary}
\newcommand{\id}{\mathrm{id}}

\DeclareMathOperator*{\Indep}{{\,\perp\mkern-12mu\perp\,}}

\DeclareMathOperator*{\given}{|}

\DeclareMathOperator{\doit}{do}

\newcommand{\RN}{\mathbb{R}}

\usepackage{tikz}
\usetikzlibrary{shapes,arrows,fit,chains,matrix}
\usetikzlibrary{positioning, shapes.geometric}
\tikzstyle{node} = [circle, minimum size = 18mm, thick, draw =black!80]
\tikzstyle{nodeinter} = [rectangle, minimum size = 16mm, thick, draw =black!80, fill=gray!30]
\tikzstyle{nodeinterwhite} = [rectangle, minimum size = 16mm, thick, draw =black!80]
\tikzstyle{nodeobserved} = [circle, minimum size = 18mm, thick, draw =black!80, fill=gray!30]
\tikzstyle{box} = [rectangle, draw =black!0]
\tikzstyle{arrow} = [thick,->,>=stealth,line width=0.6mm]
\tikzstyle{arrow2} = [dashed,->,>=stealth,line width=0.6mm]

\DeclareSymbolFont{bbold}{U}{bbold}{m}{n}
\DeclareSymbolFontAlphabet{\mathbbold}{bbold}
\newcommand\ind{\mathbbold{1}}
\DeclareMathOperator*{\sep}{\perp}

\newcommand\abs[1]{\left|#1\right|}

\newcommand\HSIC{\mathrm{HSIC}}

\newcommand\B[1]{\bm{#1}}
\newcommand\C[1]{\mathcal{#1}}

\newcommand\eref[1]{(\ref{#1})}
\usepackage{amsfonts}
\usepackage{amssymb}
\usepackage{bm}
\usetikzlibrary{matrix,decorations.pathreplacing, calc, positioning}

\tikzstyle{var}=[circle,draw=black,fill=white,thick,minimum size=20pt,inner sep=0pt]
\tikzstyle{varh}=[circle,draw=gray,fill=white,thick,minimum size=20pt,inner sep=0pt,dashed]
\tikzstyle{arr}=[->,>=stealth',draw=black,thick]
\tikzstyle{arrh}=[->,>=stealth',draw=gray,fill=gray,thick,dashed]
\tikzstyle{biarr}=[<->,>=stealth',draw=black,fill=black,thick]
\tikzstyle{biarrh}=[<->,>=stealth',draw=gray,fill=gray,thick,dashed]
\tikzstyle{noarr}=[draw=black,fill=black,thick]
\tikzstyle{noarrh}=[draw=gray,fill=gray,thick,dashed]
\tikzstyle{fac}=[rectangle,draw=black!50,fill=black!20,thick,minimum size=10pt] %or 15pt
\tikzstyle{varc}=[rectangle,draw=black,fill=white,thick,minimum size=20pt,inner sep=0pt]
\clearauthor{%
 \Name{Maximilian Ilse} \Email{ilse.maximilian@gmail.com}\\
 \addr University of Amsterdam
 \AND
 \Name{Patrick Forr\'e} \Email{p.d.forre@uva.nl}\\
 \addr University of Amsterdam
 \AND
 \Name{Max Welling} \Email{ welling.max@gmail.com}\\
 \addr University of Amsterdam
 \AND
 \Name{Joris M.~Mooij} \Email{j.m.mooij@uva.nl}\\
 \addr University of Amsterdam
}

\begin{document}

\maketitle

\begin{abstract}%
  Unobserved confounding is one of the main challenges when estimating causal effects. We propose a causal reduction method that, given a causal model, replaces an arbitrary number of possibly high-dimensional latent confounders with a single latent confounder that takes values in the same space as the treatment variable, without changing the observational and interventional distributions the causal model entails. This allows us to estimate the causal effect in a principled way from combined data without relying on the common but often unrealistic assumption that all confounders have been observed. We apply our causal reduction in three different settings. In the first setting, we assume the treatment and outcome to be discrete. The causal reduction then implies bounds between the observational and interventional distributions that can be exploited for estimation purposes. In certain cases with highly unbalanced observational samples, the accuracy of the causal effect estimate can be improved by incorporating observational data. Second, for continuous variables and assuming a linear-Gaussian model, we derive equality constraints for the parameters of the observational and interventional distributions. Third, for the general continuous setting (possibly nonlinear and non-Gaussian), we parameterize the reduced causal model using normalizing flows, a flexible class of easily invertible nonlinear transformations. We perform a series of experiments on synthetic data and find that in several cases the number of interventional samples can be reduced when adding observational training samples without sacrificing accuracy.%
\end{abstract}

\begin{keywords}%
  causal inference, causal effect estimation, dose-response curve, latent confounding, combining observational and interventional data%
\end{keywords}

\section{Notation}
We decided to make use of the standard notation in the causal Bayesian network literature,
which obfuscates the difference between a conditional distribution and a Markov kernel
(often called a conditional probability table in the discrete case). For simplicity of the following
exposition, assume discrete variables (for the continuous case similar arguments hold but a rigorous
measure-theoretic treatment gets more technical). While Markov kernels $p(Y||X)$ are defined for all
possible inputs $x$, a conditional distribution $p(Y|X)$ is only well-defined for values 
$X=x$ with $p(x)>0$.
The quantity $p(y|x,w)$ is defined as a Markov kernel in Equation \ref{eq:red-3}, which is clearly well-defined for all
values of $(x,w)$ including those where $x$ differs from $w$.

\section{Introduction}
In this work, we propose a novel, principled approach for causal effect estimation that can efficiently combine observational and interventional samples, even in the presence of unobserved confounding in the observational regime. We show that this method can potentially reduce the required Randomized Controlled Trial (RCT) sample size when sufficient observational samples are available.
% Recent real-world examples that could benefit from such an approach are the COVID-19 vaccine trials. Several of the vaccines require multiple dosages. However, due to many factors that are beyond our control, the times between administering the dosages differ. A question that then arises is: What is the effect of the time between, say, the second and the third dosage on the vaccine efficacy? In the absence of any large randomized controlled trials that provide a definite answer to this question, one may hope to estimate this by combining the few available clinical trial data with large quantities of observational data collected as a part of the different vaccination campaigns performed worldwide. The method we propose here provides a principled approach for such causal inference problems.
A complication when estimating causal effects is the potential presence of observed and---in particular---unobserved \emph{confounders} (common causes of the cause and the effect). Our key technical contribution is a construction that typically \emph{reduces} the size of the latent confounder space in a causal model. This \emph{causal reduction} operation shows that without loss of generality, one only needs to model a single latent confounding variable that takes on values in the same space as the treatment variable, even if, in reality, there could be many latent confounders and their joint value space could be much larger. This ``reduced confounder'' suffices to model exactly the observational distribution of treatment and outcome, as well as the distribution of outcome under any perfect intervention on treatment.
%The causal reduction construction makes only two key assumptions: the absence of causal feedback from outcome to treatment, and that the data was not subject to selection bias due to implicit conditioning on effects of treatment or outcome.
The causal reduction facilitates a parsimonious joint parameterization of the observational and interventional distributions. We apply causal reductions to causal inference from combined observational and interventional data in three different settings: (i) discrete treatment and outcome, (ii) continuous treatment and outcome assuming a linear-Gaussian model, (iii) continuous treatment and outcome assuming a possibly nonlinear, non-Gaussian model.

For discrete treatment and outcome variables, we show in Section \ref{sec:binary:bounds} that the causal reduction trivially implies bounds between the observational and interventional distributions. For binary variables, these bounds were already known \citep{ManskiNagin:98}, and they can also be seen as special cases of the well-known instrumental variable bounds \citep{pearl_testability_2013,balke_bounds_1997} under perfect compliance. We show that these bounds are exploited straightforwardly by the maximum likelihood estimator from combined observational and interventional data. In particular, we analyze the special case where the observational samples are heavily imbalanced: a large majority was either treated or untreated. In such settings, if sufficient observational data is available, one can sometimes substantially reduce the number of interventional samples required to accurately estimate the Average Treatment Effect (ATE) or the Conditional Average Treatment Effect (CATE).
%\Joris{Open question: how do "our" bounds relate to those of \cite{bell_einstein_1964,wolfe_inflation_2019}? I believe that the Bell-CHSH inequality applies to a different setting, namely with two treatment and two outcome variables, than ours. For the inflation method, it is not clear to me what the scope of this is, but the central idea seems to be opposite to ours: whereas we reduce the latent space (``deflate''), in the inflation method it is increased.}

Next, we apply causal reductions to the case of a continuous treatment variable. For the linear-Gaussian case (that is, where all interactions are linear, and all distributions are Gaussian), we prove in 
Section \ref{sec:bounds} that our reduced parameterization implies that the observational and interventional distributions are not independent but are related by equality constraints on their parameters. 
% We do not further work out the details of what this entails for estimation here, but instead
%This complements existing work on inequality constraints in the case of discrete treatment and outcome variables (\cite{bell_einstein_1964,wolfe_inflation_2019}) \Joris{is this the right way to cite this work here?}. 
% We conjecture that such dependencies between the observational and interventional distributions hold more generally (i.e., not only in the linear-Gaussian or discrete settings) and provide empirical support for this conjecture. \Joris{or is it just a matter of regularization?}

Last, we address the more realistic general (nonlinear, non-Gaussian) setting. In Section \ref{sec:implementations}, we parameterize the reduced causal model using a flexible class of easily invertible nonlinear transformations, so-called normalizing flows \citep{tabak_family_2013, rezende_variational_2016}. In combination with our reduction technique, normalizing flows enable the use of a simple maximum-likelihood approach to estimate the reduced model parameters. On simulated data, we observe that parameter sharing allows one to learn a more accurate model from a combination of data than from each subset individually. It is not understood at present whether this is due to constraints between the observational and interventional distributions similar to those that arise in the other settings, or due to the likelihood term involving the observational data acting as a regularizer for the interventional data. Nonetheless, we consider the empirical results as encouraging.

% In summary, our four main contributions are:
% (i) A causal reduction method that replaces arbitrary latent confounders with a single latent confounder that takes values in the same space as the treatment variable, without changing the observational and interventional distributions entailed by the causal model;
% (ii) An analysis of the inequality constraints between the observational and interventional distributions for the discrete case, and a scenario in which these constraints can be successfully exploited for causal effect estimation from combined data;
% (iii) A derivation of equality constraints between interventional and observational distributions entailed by linear Gaussian causal models;
% (iv) A flexible parameterization of the general reduced model with continuous treatment and outcome using normalizing flows, enabling joint estimation of the observational and interventional distributions from observational and interventional data without making strong parametric assumptions.

\section{Related work}
Prior work on causal inference from multiple datasets can be roughly divided into addressing three different tasks: (i) identifying causal effects when the causal structure is known (see, e.g., \cite{LeeCorreaBareinboim2020} and references therein), (ii) discovering/learning the causal structure (see, e.g., \cite{Mooij++_JMLR_20} and references therein), and (iii) estimating causal effects. The present work addresses the latter task.
Randomized Controlled Trials (RCTs) are the gold standard for estimating a causal effect \citep{fisher_1925}. The goal of an RCT is to remove all confounding biases by randomization. If we have data from a perfect RCT with no non-compliance, we can directly compute the causal effect. However, in practice, it is often costly, risky, unethical, or simply impossible to perform an RCT.

Therefore, a plethora of work on estimating causal effects solely from observational data exists. The vast majority of proposed methods assume that a set of observed variables can be used to adjust for all confounding factors \citep{gelman_propensity_2005,colnet_causal_2020}. Unfortunately, one can typically not test this assumption, and the reliability of the conclusions of such observational studies is debated \citep{Madigan++2014}.

Researchers recently started combining those different modalities to address the limitations of causal effect estimation from interventional or observational data alone. Most prior work on this topic still relies on the assumption that all confounders are observed in the observational regime (e.g., \cite{Silva2016,rosenman_propensity_2018}). Only very few approaches allow for latent confounding in the observational regime. The method by \cite{rosenman_combining_2020} attempts to reduce confounding bias rather than completely removing it. For binary treatments, \cite{kallus_removing_2018} rely on an additional assumption that the hidden confounder has a certain parametric structure that can be modelled effectively (which may also reduce confounding bias). In contrast, \cite{athey_combining_2020} depend on observed short-term and long-term outcome variables. In contrast, our approach allows us to correct confounding bias without making such additional assumptions.
Another approach that sidesteps the strong untestable assumption of no unobserved confounding is to \emph{bound} the causal effect in terms of properties of observational data \citep{balke_bounds_1997,ManskiNagin:98,pearl_testability_2013}. Recently, \cite{wolfe_inflation_2019} introduced a technique called inflation that can be used to derive bounds. %While these bounds are valid in the presence of arbitrary unobserved confounding, they are often too loose to be of practical relevance and only hold for discrete treatment variables. 
While these bounds are typically valid in the presence of arbitrary unobserved confounding, we know of no work so far that exploits such bounds to obtain better estimates of the causal effect from combined data.
Last, \cite{zhang_partial_2022} propose an operation that turns an SCM into a `canonical' form that remains counterfactually equivalent but may have a reduced latent space. Our `causal reduction' operation, on the other hand, is only designed to preserve interventional equivalence, which can be achieved in general with a smaller latent space.\footnote{For example, assuming the causal graph in Figure \ref{fig:graphical_explanation} (f), with binary treatment and effect, $\mathbf{W}$ needs to have 8 discrete states for counterfactual equivalence, where we show that for interventional equivalence we only need 2 states, i.e. $\mathbf{W}$ can be binary.}

% Furthermore, methods that do not rely on bounds or an adjustment set have to make other untestable assumptions on the causal mechanism. For example, \cite{angrist_identification_1996,Kilbertus++2020,Gunsilius2020} rely on the existence of instrumental variables that are not affected by unobserved confounders and on restrictions of the model space. \cite{miao_identifying_2018} and  \cite{Louizos++_NIPS_17} assume proxy variables that, while being correlated with unobserved confounders, do not confound the treatment and outcome themselves. Last, the deconfounder of \cite{wang_blessings_2019} builds on the assumptions that there are no unobserved single-cause confounders.

\section{Causal Reductions}\label{sec:reduction}
We introduce our key technical contribution that we refer to as a \emph{causal reduction}. 
This is a simple procedure to replace any number of confounders by a single confounder that takes on values in the same space as the treatment variable, while preserving the important causal semantics of the model, in particular, the observational distribution and the causal effect of treatment on outcome.
While we start by considering a simple setting with only two observed variables, we extend this in Appendix~\ref{sec:reduction_with_confounders} to account for additional observed confounders.
The basic construction procedure can possibly be applied in more general settings as well.
For the measure-theoretic justification of our derivation we refer to \citep{Forre2021}.
%For simplicity of exposition, we will make some assumptions regarding the types of variables below, but the construction of the causal reduction can be done for any standard measurable spaces. 
While we choose to make use of the formalism of causal Bayesian networks to derive the result, similar results can be obtained easily in other causal modeling frameworks, such as structural causal models and the potential-outcome framework, see Appendix \ref{app:potential_outcomes}.

\subsection{Reduction without observed confounders}
\label{sec:reduction_without_confounders}

Consider a treatment variable $\mathbf{X}\in \mathcal{X}$
% = \mathbb{R}^M$ 
and an outcome variable $\mathbf{Y}\in \mathcal{Y}$.
% = \mathbb{R}^N$.
The spaces $\mathcal{X}$ and $\mathcal{Y}$ can be any standard measurable space, for example, $\mathcal{X} = \mathbb{R}^M$ and $\mathcal{Y} = \mathbb{R}^N$ in the continuous setting, or $\mathcal{X} = \{1,\dots,m\}$ and $\mathcal{Y} = \{1,\dots,n\}$ in the discrete setting, with binary treatment ($m=2$) and a real-valued treatment variable ($M=1$) as frequently occurring special cases.
We assume that the outcome does not cause the treatment. Furthermore, let there exist $K$ latent confounders $Z_1, \dots, Z_K$, where each $Z_i \in \mathcal{Z}_i$ also in some standard measurable space, with an arbitrary dependency structure. See Figure \ref{fig:graphical_explanation} (a) for an example of a corresponding Directed Acyclic Graph (DAG) of the causal Bayesian network.\footnote{Alternatively, one can interpret this as the graph of a Structural Causal Model, see for example \citep{Bongers++_AOS_21}.} Without loss of generality, we can summarize the $K$ latent confounders $Z_1, \dots, Z_K$ with arbitrary dependency structure using a single latent confounder $\mathbf{Z} \in \mathcal{Z} = \prod_{k=1}^K \mathcal{Z}_k$: 
$p(\mathbf{x}, \mathbf{y}) = \int_{\mathcal{Z}_1} \dots \int_{\mathcal{Z}_K} p(\mathbf{x},\mathbf{y}, z_1, \dots, z_K) dz_1 \dots dz_K = \int_\mathcal{Z} p(\mathbf{x,y,z}) d\mathbf{z}.$
The resulting causal Bayesian network is shown in Figure \ref{fig:graphical_explanation} (b), which has the following factorization:
\begin{align}
    p(\mathbf{x,y,z}) &= p(\mathbf{y}|\mathbf{x},\mathbf{z}) p(\mathbf{x}|\mathbf{z}) p(\mathbf{z}). \label{eq:red-1}
\end{align}
We aim to replace the above causal Bayesian network with one that is interventionally equivalent with respect to perfect interventions on any subset of $\{\mathbf{X},\mathbf{Y}\}$, but where the latent confounder space $\mathcal{Z}$ can be smaller. In particular, this means that we will preserve the observational distribution $p(\mathbf{x},\mathbf{y})$ and the ``causal effect'' $p(\mathbf{y}|\doit(\mathbf{x}))$, that is, the distribution of $\mathbf{Y}$ for any perfect intervention on $\mathbf{X}$ that sets it to a value $\mathbf{x} \in \mathcal{X}$.

\begin{figure}[h]
\begin{center}
\hspace{-0.8cm}
\begin{subfigure} % {.5\textwidth}
\centering
\scalebox{0.45}{\begin{tikzpicture}
\LARGE
\node [node](z2) {$Z_2$};
\node [node, left = of z2](z1) {$Z_1$};
\node [node, right = of z2](zd) {$Z_K$};
\node [text width=3cm, left = of zd, xshift=3.4cm](dots) {$\mathbf{\dots}$};
\node [node, below right = of z2, xshift=-0.65cm](z3) {$Z_3$};
\node [node, below left = of z2, xshift=0.65cm](z4) {$Z_4$};
\node [nodeobserved, below = of z3](x) {$\mathbf{Y}$};
\node [nodeobserved, below = of z4](y) {$\mathbf{X}$};

\draw [arrow] (z1) -- (z2);
\draw [arrow] (z1) -- (z3);
\draw [arrow] (z2) -- (z3);
\draw [arrow] (z3) -- (z4);
\draw [arrow] (zd) -- (z4);
\draw [arrow] (zd) -- (z3);
\draw [arrow] (z4) -- (y);
\draw [arrow] (z4) -- (x);
\draw [arrow] (z3) -- (x);
\draw [arrow] (y) -- (x);
\draw [arrow] (zd) to [out=-90,in=45] (x);
\end{tikzpicture}}
\end{subfigure}
\hspace{-0.8cm}
\begin{subfigure} % {.5\textwidth}
    \centering
\scalebox{0.45}{\begin{tikzpicture}
\LARGE
\node [node](z) {$\mathbf{Z}$};
\node [nodeobserved, below left = of z, yshift=-0.25cm, xshift=0.65cm](x) {$\mathbf{X}$};
\node [nodeobserved, below right = of z, yshift=-0.25cm, xshift=-0.65cm](y) {$\mathbf{Y}$};

\draw [arrow] (z) -- (y);
\draw [arrow] (z) -- (x);
\draw [arrow] (x) -- (y);
\end{tikzpicture}}
% \subcaption{}
\end{subfigure}
\begin{subfigure} % {.5\textwidth}
    \centering
\scalebox{0.45}{\begin{tikzpicture}
\LARGE
\node [nodeobserved] (y) {$\mathbf{Y}$};
\node [node, above = of y](z) {$\mathbf{Z}$};
\node [nodeobserved, double, left = of y](x) {$\mathbf{X}$};
\node [node, left = of z](w) {$\mathbf{W}$};
\draw [arrow] (w) -- (x);
\draw [arrow] (x) -- (y);
\draw [arrow] (z) -- (w);
\draw [arrow] (z) -- (y);
\end{tikzpicture}}
% \subcaption{}
\end{subfigure}
\end{center}
\begin{center}
\begin{subfigure} % {.5\textwidth}
    \centering
\scalebox{0.45}{\begin{tikzpicture}
\LARGE
\node [node](z) {$\mathbf{W, Z}$};
\node [nodeobserved, double, below left = of z, yshift=-0.25cm, xshift=0.65cm](x) {$\mathbf{X}$};
\node [nodeobserved, below right = of z, yshift=-0.25cm, xshift=-0.65cm](y) {$\mathbf{Y}$};

\draw [arrow] (z) -- (y);
\draw [arrow] (z) -- (x);
\draw [arrow] (x) -- (y);
\end{tikzpicture}}
% \subcaption{}
\end{subfigure}
\begin{subfigure}% {.5\textwidth}
    \centering
\scalebox{0.45}{\begin{tikzpicture}
\LARGE
\node [nodeobserved] (y) {$\mathbf{Y}$};
\node [node, above = of y](z) {$\mathbf{Z}$};
\node [nodeobserved, double, left = of y](x) {$\mathbf{X}$};
\node [node, left = of z](w) {$\mathbf{W}$};
\draw [arrow] (w) -- (x);
\draw [arrow] (x) -- (y);
\draw [arrow] (w) -- (z);
\draw [arrow] (z) -- (y);
\end{tikzpicture}}
% \subcaption{}
\end{subfigure}
\begin{subfigure} % {.5\textwidth}
    \centering
\scalebox{0.45}{\begin{tikzpicture}
\LARGE
\node [node](z) {$\mathbf{W}$};
\node [nodeobserved, double, below left = of z, yshift=-0.25cm, xshift=0.65cm](x) {$\mathbf{X}$};
\node [nodeobserved, below right = of z, yshift=-0.25cm, xshift=-0.65cm](y) {$\mathbf{Y}$};

\draw [arrow] (z) -- (y);
\draw [arrow] (z) -- (x);
\draw [arrow] (x) -- (y);
\end{tikzpicture}}
% \subcaption{}
\end{subfigure}
\end{center}

    \caption[A graphical explanation of our causal reduction technique.]{A graphical explanation of our causal reduction technique. We use a double circle to indicate that a variable is a deterministic function of its parents.
    % (a) We assume a treatment variable $\mathbf{X}$, an outcome variable $\mathbf{Y}$, and $K$ latent confounders $Z_1, \dots, Z_K$ with an arbitrary dependency structure. (b) We represent the $K$ latent confounders $Z_1, \dots, Z_K$ by $\mathbf{Z} \in \mathcal{Z}$. (c) We create a copy of $\mathbf{X}$ called $\mathbf{W}$. We use a double circle to indicate that a variable is a deterministic function of its parents. (d, e) Instead of using the factorization from (c), $p(\mathbf{w, z})=p(\mathbf{w}|\mathbf{z})p(\mathbf{z})$, we choose $p(\mathbf{w, z})=p(\mathbf{z}|\mathbf{w})p(\mathbf{w})$. (f) Last, we marginalize over $\mathbf{Z}$. Note that at every step (a--f) the Bayesian networks are interventionally equivalent with respect to perfect interventions on any subset of $\{\mathbf{X}$, $\mathbf{Y}\}$. In particular, they all induce the same observational distribution $p(\mathbf{x},\mathbf{y})$ and interventional distributions $p(\mathbf{y}|\doit(\mathbf{x}))$.
    }
    \label{fig:graphical_explanation}
\end{figure}
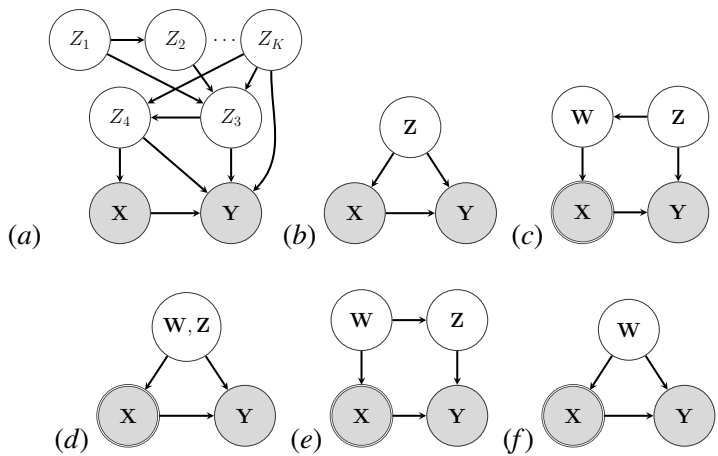

First, we generate a copy $\mathbf{W}:=\mathbf{X}$ of the treatment variable $\mathbf{X}$. We will interpret $\mathbf{W}$ as a latent variable and $\mathbf{X}$ as an observed deterministic effect of $\mathbf{W}$ via the identity function $\mathbf{X}=\id(\mathbf{W})$.
We obtain the Bayesian Network in Figure \ref{fig:graphical_explanation} (c):
 \begin{align}
    p(\mathbf{x,y,w,z}) &= p(\mathbf{y}|\mathbf{x},\mathbf{z})  p(\mathbf{x}|\mathbf{w})  p(\mathbf{w}|\mathbf{z})  p(\mathbf{z}), \label{eq:red-2}
\end{align}
where $p(\mathbf{w}|\mathbf{z}):= p(\mathbf{x}|\mathbf{z})|_{\mathbf{x}=\mathbf{w}}$ is a copy of the Markov kernel that appears as a factor in Equation~\ref{eq:red-1}, but evaluated in $\mathbf{w}$ rather than in $\mathbf{x}$. Furthermore, $p(\mathbf{x}|\mathbf{w}):=\delta_\mathbf{w}(\mathbf{x})$ is the Dirac measure centered at $\mathbf{w}$, representing the deterministic identity map from $\mathbf{W}$ to $\mathbf{X}$. If we marginalize out $\mathbf{W}$ we arrive at the initial causal Bayesian network in Figure \ref{fig:graphical_explanation} (b) again. Since interventions on observed variables commute with marginalizing over latent variables \citep{Bongers++_AOS_21}, the Bayesian networks in Figure \ref{fig:graphical_explanation} (b) and (c) are interventionally equivalent with respect to perfect interventions on any subset of $\{\mathbf{X},\mathbf{Y}\}$. 

Second, we refactorize the latent distribution as shown in Figure  \ref{fig:graphical_explanation} (c), (d) and (e):
\begin{align}
    p(\mathbf{x,y,w,z}) &= p(\mathbf{y}|\mathbf{x},\mathbf{z})  p(\mathbf{x}|\mathbf{w})  p(\mathbf{w}|\mathbf{z})  p(\mathbf{z}) \label{eq:red-3}\\
               &= p(\mathbf{y}|\mathbf{x},\mathbf{z})  p(\mathbf{x}|\mathbf{w})  p(\mathbf{w,z}) \label{eq:red-4}\\
               &= p(\mathbf{y}|\mathbf{x},\mathbf{z})  p(\mathbf{x}|\mathbf{w})  p(\mathbf{z}|\mathbf{w})  p(\mathbf{w}). \label{eq:red-5}
\end{align}
The Bayesian networks representing these three factorizations are interventionally equivalent with respect to perfect interventions on any subset of $\{\mathbf{X},\mathbf{Y}\}$, as we only factor the latent distributions differently (and do not consider interventions on the latent variables).

Last, we can marginalize over $\mathbf{Z}$ and obtain:
\begin{align}
    p(\mathbf{x,y,w}) &= p(\mathbf{y}|\mathbf{x},\mathbf{w})  p(\mathbf{x}|\mathbf{w})  p(\mathbf{w}), \label{eq:red-6}
\end{align}
where we used the following composed Markov kernel:
\begin{align}
        p(\mathbf{y}|\mathbf{x},\mathbf{w}) &:= \int p(\mathbf{y}|\mathbf{x},\mathbf{z})  p(\mathbf{z}|\mathbf{w}) \,d\mathbf{z}. \label{eq:red-7}
\end{align}
Again, since marginalizing over latent variables and interventions on observed variables commute, the final Bayesian network in Figure \ref{fig:graphical_explanation} (f) is interventionally equivalent to the ones in Figure \ref{fig:graphical_explanation} (a--e) with respect to perfect interventions on any subset of $\{\mathbf{X},\mathbf{Y}\}$.

Since $\mathbf{W}$ is a copy of $\mathbf{X}$, we successfully replaced the latent confounder space $\mathcal{Z} = \prod_{k=1}^K \mathcal{Z}_k$ with $\mathcal{X}$. In case $\mathcal{Z} = \RN^K$ and $\mathcal{X} = \RN^M$ and $M < K$, the dimensionality of the latent space will be reduced. 
In the common case of a one-dimensional $\mathbf{X}$, we expect $M = 1 \ll K$ and therefore achieve a significant reduction of the latent space. 
In the discrete case, the cardinality of $\mathcal{X}$ may be much lower than that of $\mathcal{Z}$.
For example, in case treatment is binary, a single binary confounder in the model suffices. We formulate the conclusion as a theorem in Appendix \ref{sec:causal_red_thm}. Last, in Appendix \ref{sec:reduction_with_confounders}, we derive a reduced causal model for the case of additional observed confounders.

\subsection{Replacing conditional distributions by functions}
\label{sec:add_latents}
Whereas previously, we made use of causal Bayesian networks to formulate our reduction operation, we now move to Structural Causal Models (SCMs) \citep{pearl_causality_nodate,Bongers++_AOS_21} in order to obtain convenient parameterizations in the non-parametric setting. We make use of the exogenous variables $\mathbf{U,V}$ to represent the noise in the reduced causal model. This, in turn, allows us to express all causal relationships as deterministic functions. Estimating the model then boils down to estimating these functions, as we will illustrate in Section~\ref{sec:implementations}.

\begin{theoremEnd}{Thm}
\label{gaus_plus_map}
Let $\mathbf{Y}$ be a `conditional' $\mathbb{R}^N$-valued random variable with Markov kernel $P(\mathbf{Y}|\mathbf{X})$ and input $\mathbf{X}$ that can take values in any measurable space. 
Then there exists an $N$-dimensional standard normal random variable $\mathbf{V} \sim \mathcal{N}(\mathbf{0},\mathbf{I}_N)$ independent of $\mathbf{X}$ and a deterministic measurable map $F$ such that:
\begin{align}
    \mathbf{Y} = F(\mathbf{V}, \mathbf{X})\quad \text{a.s.}
\end{align}
Furthermore, the map $F$ is `well-behaved', in the sense that it is composed out of (inverse) conditional cumulative distribution functions.
\end{theoremEnd}
\begin{proofEnd}
We use Theorem \ref{cdf-ind} inductively. 
\begin{enumerate}
    \item Consider the cqf $F_1$ of $P(Y_1|\mathbf{X})$. Then by \ref{cdf-ind} there is a random variable $E_1 \sim U[0,1]$ independent of $\mathbf{X}$ such that $Y_1=F_1(E_1|\mathbf{X})$ a.s.
    \item Now consider the cqf $F_2$ of $P(Y_2|E_1, \mathbf{X})$. Then by \ref{cdf-ind} there is a random variable $E_2 \sim U[0,1]$ independent of $E_1$, $\mathbf{X}$ such that $Y_2 = F_2(E_2|E_1, \mathbf{X})$ a.s.
    \item Now consider the cqf $F_3$ of $P(Y_3|E_2,E_1, \mathbf{X})$. Then by \ref{cdf-ind} there is a random variable $E_3 \sim U[0,1]$ independent of $E_2,E_1$, $\mathbf{X}$ such that $Y_3 = F_3(E_3|E_2,E_1, \mathbf{X})$ a.s.
    \item and so on .... until:
    \item $Y_N = F_N(E_N|E_{N-1},\dots,E_1, \mathbf{X})$ a.s. with $E_N \sim U[0,1]$ independent of $E_{N-1},\dots,E_1$, $\mathbf{X}$.
\end{enumerate}
Now we put $Z_d:= \Phi^{-1}(E_d)$, where $\Phi$ is the cdf of $\mathcal{N}(0,1)$. Then $E_d = \Phi(Z_d)$ and the $Z_d$ are $\mathcal{N}(0,1)$-distributed and  $\mathbf{Z}=(Z_1,\dots,Z_N)$ is independent of $\mathbf{X})$. So $\mathbf{Z}=(Z_1,\dots,Z_N) \sim \mathcal{N}(\mathbf{0},\mathbf{I}_N)$ and independent of $\mathbf{X}$.
Furthermore, we have almost surely the equations:
\begin{align}
    Y_1 & = F_1(\Phi(Z_1)|\mathbf{X}),\\
    Y_2 & = F_2(\Phi(Z_2)|\Phi(Z_1), \mathbf{X}),\\
    \vdots &\qquad\qquad\qquad \ddots \\
    Y_N & = F_M(\Phi(Z_N)|\Phi_{(Z_{N-1})},\dots,\Phi(Z_1), \mathbf{X}).
\end{align}
This shows the claim.
\end{proofEnd}
Applying Theorem~\ref{gaus_plus_map} twice gives a reduced SCM from the reduced causal Bayesian network in Equation \ref{eq:red-6} with structural equations
\begin{align}
  \mathbf{X} & = F(\mathbf{U}), \label{eq:reduced_se_X}\\
  \mathbf{Y} & = G(\mathbf{U, V, X}), \label{eq:shared_function}
\end{align}
where $\mathbf{U} \sim \mathcal{N}(\mathbf{0, I}_M), \mathbf{V} \sim \mathcal{N}\left(\mathbf{0, I}_N\right)$ and $\mathbf{U} \Indep \mathbf{V}$, and $F, G$ are  deterministic maps. This SCM encodes the same observational distribution $p(\mathbf{x},\mathbf{y})$ and interventional distributions $p(\mathbf{y} | \doit(\mathbf{x})),$ $p(\mathbf{x} | \doit(\mathbf{y}))$ as the causal Bayesian network. This allows us to ``parameterize'' the reduced causal model in terms of the two functions $F$ and $G$.

\section{Parameter estimation in the discrete case}
\label{sec:binary:bounds}

In this section we present the first application of our causal reduction technique.
We formulate the maximum likelihood estimator of the causal effect from observational and interventional data, considering the case in which all observed variables are discrete.
The exact maximum likelihood estimator exploits that the parameters of the two prediction tasks are shared. 

\subsection{Maximum likelihood estimation}
To formulate the corresponding maximum likelihood estimation problem, we introduce real-valued parameters
\begin{equation*}
  \begin{cases}
    \phi_w := p(w) & \text{for $w \in \C{X}$} \\
    \vartheta_{y\mid x} := p(Y=y \mid X=x) & \text{for $x\in\C{X}, y\in\C{Y}$} \\
    \psi_{y\mid x} := p(y \mid \doit(x)) & \text{for $x\in\C{X}, y\in\C{Y}$.}
  \end{cases}
\end{equation*}
The parameters $\phi$ and $\vartheta$ parameterize the observational distribution, while $\psi$ directly parameterizes the interventional distributions.
These parameters are subject to the following constraints:
\begin{align}
    \forall &x \in \C{X}: 0 \le \phi_x \le 1, \qquad \sum_{x\in\C{X}} \phi_x = 1\\
\forall &x\in\C{X} \ \forall y\in\C{Y}: 0 \le \vartheta_{y \mid x} \le 1, \qquad  \forall x \in \C{X}: \sum_{y \in\C{Y}} \vartheta_{y \mid x} = 1\\
  \forall &x\in\C{X} \ \forall y\in\C{Y}: \phi_x \vartheta_{y \mid x} \le \psi_{y,x} \le \phi_x \vartheta_{y \mid x} + 1 - \phi_x, \qquad \forall x \in \C{X}: \sum_{y\in\C{Y}} \psi_{y\mid x} = 1. \label{eq:discrete_causal_effect_bound_parameters}
\end{align}
Apart from the obvious nonnegativity and normalization constraints, these also include the bounds of \cite{ManskiNagin:98} in Equation~\ref{eq:discrete_causal_effect_bound} in Appendix \ref{sec:manski}.
Given an observational dataset $(x^O_1,y^O_1),\dots,(x^O_{N_O},y^O_{N_O})$ and an interventional dataset
$(x^I_1,y^I_1),$ $\dots,(x^I_{N_I},y^I_{N_I})$, we obtain the total log-likelihood in terms of the parameters:
\begin{equation}\label{eq:joint-log-likelihood-discrete}\begin{split}
  \ell(\phi,\vartheta,\psi; x^{I},y^{I},x^{O},y^{O}) 
  &= \sum_{j=1}^{N_I} \log p(y^{I}_j \mid \doit(x^{I}_j)) + \sum_{j=1}^{N_O} \log p(x^O_j,y^O_j) \\
  &= \sum_{j=1}^{N_I} \log \psi_{y^{I}_j,x^{I}_j} + \sum_{j=1}^{N_O} \big( \log \phi_{x^O_j} + \log \vartheta_{y^O_j \mid x^O_j} \big).
\end{split}\end{equation}
This decomposes as a sum of two terms, the first term only involving the interventional data and the parameters $\psi$, the second term only involving the observational data and the parameters $\vartheta, \phi$.
However, importantly, the two problems are not independent, as the parameters have to satisfy the constraints in Equation~\ref{eq:discrete_causal_effect_bound_parameters}.

In principle, one can solve for the maximum likelihood estimator by using the technique of Lagrange multipliers to take into account the (equality and inequality) constraints on the parameters.
We will not do so here as the calculation seems to become rather cumbersome, and it does not seem possible to obtain an analytical closed-form expression for the maximum likelihood estimator.
However, standard numerical optimization techniques can be applied to compute the maximum likelihood estimator numerically.

\subsection{Infinite observational data}

To gain some intuition, we will instead consider the limit that the interventional data size $N_I$ is fixed, whereas the observational data size $N_O \to \infty$. 
In that limit, the observational part of the log-likelihood dominates, and we can easily optimize this with respect to $\vartheta$ and $\phi$ to obtain:
$$\hat \phi_{x} = \frac{N^O_{x+}}{N^O_{++}}, \qquad \hat \vartheta_{y \mid x} = \frac{N^O_{xy}}{N^O_{x+}},$$
where $N^O_{++} := \sum_{x\in\C{X}} N^O_{x+}$, with $N^O_{x+} := \sum_{y\in\C{Y}} N^O_{xy}$, with $N^O_{xy} := \sum_{j=1}^{N_O} \ind_{x}(x^O_j) \ind_{y}(y^O_j)$.
Given these (asymptotically) optimal $\hat\vartheta, \hat\phi$ we can now maximize the interventional part of the log-likelihood with respect to $\psi$, where we have to bear in mind the constraints in Equation~\ref{eq:discrete_causal_effect_bound_parameters}.
%\begin{equation*}
%\hat\phi_x \hat\vartheta_{y \mid x} \le \psi_{y,x} \le \hat\phi_x \hat\vartheta_{y \mid x} + 1 - \hat\phi_x.
%\end{equation*}
The optimum for $\psi_{y,x}$ will either be at the lower or the upper bound, or in between. 
Ignoring the bounds, the optimum would be taken at $\frac{N^I_{y|x}}{N^I_{+|x}}$. 
If this falls below the lower bound, or above the upper bound, the value will be clipped at the respective bound.
Hence,
\begin{equation}\label{eq:approx_ML_estimator}
\hat\psi_{y,x} = \begin{cases}
  \hat\phi_x \hat\vartheta_{y \mid x} & \frac{N^I_{y|x}}{N^I_{+|x}} < \hat\phi_x \hat\vartheta_{y \mid x} \\
  \frac{N^I_{y|x}}{N^I_{+|x}} & \hat\phi_x \hat\vartheta_{y \mid x} \le \frac{N^I_{y|x}}{N^I_{+|x}} \le \hat\phi_x \hat\vartheta_{y \mid x} + 1 - \hat\phi_x \\
  \hat\phi_x \hat\vartheta_{y \mid x} + 1 - \hat\phi_x & \hat\phi_x \hat\vartheta_{y \mid x} + 1 - \hat\phi_x < \frac{N^I_{y|x}}{N^I_{+|x}}. 
\end{cases}\end{equation}
Since $N^O$ is large, the observational parameters $\phi, \vartheta$ are estimated with high accuracy.
  This can help improve the accuracy in the estimated interventional parameters $\psi$ by means of the bound, especially in cases where only few interventional data points are available (i.e., if some $N^I_{+|x} \ll N^O_{+|x}$). 
For larger values of $N^I$, the estimator $\hat\psi_{y,x}$ will ignore the observational data, and reduce to the standard causal effect estimator from interventional data, $\frac{N^I_{y|x}}{N^I_{+|x}}$.

\subsection{Case study: mostly untreated observational patient population}

We now consider one specific example to get an impression of when this estimator may yield a benefit, and how large the benefit can be.
We assume treatment and outcome to be binary.
We can then focus on estimating the \emph{average treatment effect (ATE)},
$$\tau := p(Y=1 \mid \doit(X=1)) - p(Y=1 \mid \doit(X=0)).$$
The maximum likelihood (plug-in) estimator for this is just
$$\hat \tau^{IO} := \hat \psi_{1,1} - \hat \psi_{1,0},$$
where we used the superscript ``$IO$'' to indicate that this estimator uses interventional and observational data.
The classical estimator of the ATE based on interventional data only is
$$\hat \tau^I := \frac{N^I_{1|1}}{N^I_{+|1}} - \frac{N^I_{1|0}}{N^I_{+|0}}.$$
Our statistical model has a 5-dimensional parameterization.
The bounds in \cite{ManskiNagin:98} (see Equation~\ref{eq:discrete_causal_effect_bound}) read explicitly:
\begin{align}
  p(x=0,y=1) &\le p(y=1\mid\doit(x=0)) \le 1 - p(x=0,y=0) \label{eq:bound_binary_x0}\\
  p(x=1,y=1) &\le p(y=1\mid\doit(x=1)) \le 1 - p(x=1,y=0) \label{eq:bound_binary_x1}.
\end{align}
In Appendix \ref{sec:fig_bounds} Figure~\ref{fig:bounds} we give an illustration of these bounds.

Let us consider the following scenario. 
We have a new drug and we are setting up an RCT. 
We have historical observational data where all individuals are untreated, i.e., $p(x=1)=0$.\footnote{A similar analysis pertains in in case $p(x=1)$ is small but nonzero.}
The bound in Equation~\ref{eq:bound_binary_x0} then becomes tight, yielding
\begin{align}
    p(y=1) = p(y=1\mid\doit(x=0)).
\end{align}
%\begin{align}
%0 = p(x=1,y=1) \le p(y=1\mid\doit(x=1)) \le 1 - p(x=1,y=0) = 1 \\
%p(y=1) = p(x=0,y=1) \le p(y=1\mid\doit(x=0)) \le 1 - p(x=0,y=0) = 1 - p(y=0) = p(y=1)
%\end{align}
On the other hand, the bound in Equation~\ref{eq:bound_binary_x1} becomes non-informative and can be ignored.
Thus, we can identify $p(y=1\mid\doit(x=0))$ from the observational data as $p(y=1)$.
When setting up an RCT, we therefore only need a treatment group with $\doit(x=1)$ in order to identify $p(y=1\mid\doit(x=1))$, and there is no need to include a group in which individuals are not treated.
In this way, we can reduce the number of participants of the RCT by a factor of two without losing accuracy (compared to a standard RCT with a 50\%-50\% split in treatment and control group). 
We might even gain in accuracy if $N_O$ is large, because that allows to estimate $p(y=1)$, and hence $p(y=1\mid\doit(x=0))$, and therefore $\tau$, more accurately. Note, in Appendix \ref{sec:limitations} we discuss the influence of the placebo effect on the analysis above.

\section{Parameter constraints in the linear Gaussian case}
\label{sec:bounds}
For the second application of our causal reduction, we consider the case where all causal relationships in Figure \ref{fig:graphical_explanation} (a) are linear, and all distributions are Gaussian. We can then guarantee that the reduced causal model is linear Gaussian as well.

\begin{theoremEnd}{Cor}[Reduced linear Gaussian model]
\label{cor:reduction-linear-gaussian}
 Consider a linear Gaussian SCM (or causal Bayesian network with possible latent variables) with observed variables $\mathbf{X}$ and $\mathbf{Y}$ such that $\mathbf{Y}$ is not ancestor of $\mathbf{X}$. Then this causal model is interventionally equivalent to a reduced linear Gaussian causal model with the following structural equations:
\begin{align}
    \mathbf{X} &= \mathbf{a} + B \mathbf{U},\\
    \mathbf{Y} &= \mathbf{c} + D \mathbf{X} + E \mathbf{U} + F \mathbf{V}, 
\label{eq:reduced_scm_linear}
\end{align}
with vectors $\mathbf{a}$, $\mathbf{c}$ and matrices $B$, $D$, $E$, $F$, where $B$ and $F$ can be chosen to be lower-triangular with non-negative diagonal entries, and
where $\mathbf{U}$ is a standard Gaussian latent variable of the same dimension as $\mathbf{X}$ and where $\mathbf{V}$ is a standard Gaussian latent variable of the same dimension as $\mathbf{Y}$ that is independent of $\mathbf{U}$.
\end{theoremEnd}
\begin{proofEnd}
This follows the same steps as the general construction in Equations \ref{eq:red-1},  \ref{eq:red-2}, \ref{eq:red-3}, \ref{eq:red-4}, where $p(\mathbf{x}|\mathbf{w})=\delta_w(\mathbf{x})$ reflects the identity map. 
In Equation \ref{eq:red-5}, note that $p(\mathbf{z|w})$ is linear Gaussian by the well-known conditioning formula for jointly Gaussian distributions. 
We then arrive at Equation \ref{eq:red-6}, where it can be checked that in Equation \ref{eq:red-7} both parts, $p(\mathbf{z|w})$ and $p(\mathbf{y|x,z})$, are linear Gaussian, thus makes $p(\mathbf{y|x,w})$ linear Gaussian. Finally, we use the reparameterization trick together with a Cholesky decomposition, as seen in Section \ref{sec:add_latents}, to turn $p(\mathbf{w})$ into a standard Gaussian $p(\mathbf{u})$, which makes $p(\mathbf{x|u})$, as a composition of identity map and linear Gaussian also a linear Gaussian. Note that $p(\mathbf{y|x,u})$ again is linear Gaussian by similar arguments. Last we use the reparameterization trick again to obtain $p(\mathbf{y}|\mathbf{x,u, v})$ where $\mathbf{V} \sim \mathcal{N}(\mathbf{0,I}_N)$.
\end{proofEnd} We use the reduced linear Gaussian model from Corollary \ref{cor:reduction-linear-gaussian} to prove that the parameters of the interventional distribution constrain the parameters of the observational distribution.

\begin{theoremEnd}{Thm}[Linear Gaussian parameter constraints]
\label{thm:constraints}
Consider a linear-Gaussian SCM (or causal Bayesian network with possible latent variables) with two observed variables $\mathbf{X}$ and $\mathbf{Y}$ such that $\mathbf{Y}$ is not ancestor of $\mathbf{X}$. The entailed observational and interventional distributions are Gaussian. Modeling $p(\mathbf{x}), p(\mathbf{y|x})$ and $p(\mathbf{y}|\doit(\mathbf{x}))$ independently from each other could be done with the following parameterization:
\begin{align}
\label{eq:linear_px}
    p(\mathbf{x}) & = \mathcal{N}\left(\mathbf{x}|\boldsymbol{\alpha},\Sigma\right),\\
    \label{eq:linear_py_given_x}
    p(\mathbf{y|x}) & = \mathcal{N}\left(\mathbf{y}|\boldsymbol{\gamma} + \Delta \mathbf{x}, \Pi\right),\\
    \label{eq:linear_py_given_dox}
        p(\mathbf{y}|\doit(\mathbf{x})) & = \mathcal{N}\left(\mathbf{y}|\boldsymbol{\widetilde{\gamma}} + \widetilde{\Delta} \mathbf{x}, \widetilde{\Pi}\right),
\end{align}
with covariance matrices $\Sigma$, $\Pi$, $\widetilde{\Pi}$. However, using the reduced causal model from Corollary \ref{cor:reduction-linear-gaussian} we find that these parameters are constrained by the following relations:
\begin{align}
 \left(\boldsymbol{\widetilde{\gamma}} - \boldsymbol{\gamma}\right) +\left(\widetilde{\Delta} - \Delta\right) \boldsymbol{\alpha} &=0, \label{eq:constraint1}\\
   \left(\widetilde{\Delta}-\Delta\right) \Sigma \left(\widetilde{\Delta} - \Delta\right)^\top +\Pi & = \widetilde{\Pi}. \label{eq:constraint2}
\end{align}
\end{theoremEnd}
\begin{proofEnd}
The linear version of the reduced SCM in Equation \ref{eq:reduced_scm_linear} entails the following distributions over $\mathbf{x}$ and $\mathbf{y}$
\begin{align}
    p(\mathbf{x}) &= \mathcal{N}\left(\mathbf{x}|\mathbf{a},BB^\top\right), \label{eq:linear_px_2}\\
    p(\mathbf{y}|\mathbf{x}) &= \mathcal{N}\left(\mathbf{y}| \mathbf{c} + D \mathbf{x} + E B^{-1} (\mathbf{x}-\mathbf{a}\right), FF^\top),\label{eq:linear_py_given_x_2}\\
     p(\mathbf{y}|\doit(\mathbf{x})) &  = \mathcal{N}\left(\mathbf{y}|\mathbf{c} + D \mathbf{x}, EE^\top+FF^\top\right) \label{eq:linear_py_given_dox_2},
\end{align}

Comparing Equations \ref{eq:linear_px}, \ref{eq:linear_py_given_x}, \ref{eq:linear_py_given_dox} with \ref{eq:linear_px_2}, \ref{eq:linear_py_given_x_2}, \ref{eq:linear_py_given_dox_2}  we immediately get the equations for the parameters:
\begin{align}
    \boldsymbol{\alpha} & = \mathbf{a},\\
    \Sigma  & = BB^\top, \label{eq:beta} \\
    \boldsymbol{\gamma} + \Delta \mathbf{x}  & = \mathbf{c} + 
    \left(D  + E B^{-1}\right) \mathbf{x} -  E B^{-1}  \mathbf{a}, \label{eq:lin_combo_x}\\
    \boldsymbol{\gamma} & \stackrel{\mathbf{x}=\mathbf{0}}{=} \mathbf{c} - E B^{-1}  \mathbf{a}, \label{eq:lin_combo_0}\\
    \Pi & = FF^\top,\\
    \widetilde{\boldsymbol{\gamma}} & = \mathbf{c},\\
    \widetilde{\Delta} & = D,\\
    \widetilde{\Pi} & = EE^\top+FF^\top.
\end{align}
Substituting $\mathbf{a},\mathbf{c},D,FF^\top$ and then subtracting Equation \ref{eq:lin_combo_0} from \ref{eq:lin_combo_x} and solving for all $\mathbf{x}$ we get the constraints:
\begin{align}
   \Delta & = 
    \widetilde{\Delta}+ E B^{-1},\label{eq:delta}\\
\boldsymbol{\gamma} & = \widetilde{\boldsymbol{\gamma}} - E B^{-1} \boldsymbol{\alpha}, \label{eq:gamma}\\
\widetilde{\Pi} &= \Pi + EE^\top.\label{eq:zeta}
\end{align}
With Equation \ref{eq:delta} we see that $E = (\Delta - \widetilde{\Delta}) B$, which we can just plug into Equations \ref{eq:gamma} and \ref{eq:zeta}. Finally using Equation \ref{eq:beta} to replace $BB^\top$ with $\Sigma$ in Equation \ref{eq:zeta} will give the claim.
\end{proofEnd}
From Equation \ref{eq:constraint2} we can easily see that $\widetilde{\Pi} - \Pi$ is positive semidefinite. Furthermore, we see that these constraints lead to a reduced parameter count, $N$ parameters for Equation \ref{eq:constraint1} and $N(N+1)/2$ parameters for Equation \ref{eq:constraint2}, assuming $\mathbf{y}$ to be $N$-dimensional. In total, we have reduced the parameter count by $N(N+3)/2$ by modeling the parameters of the observational and interventional distributions jointly. 

% In the linear Gaussian case, the reduced causal model tells us exactly how many parameters we need to model the observational and interventional distribution and which parameters are shared. Indeed, the parameters $\mathbf{c},D,E$ and, $F$ are shared between the observational and interventional distribution. We can estimate them jointly using observational and interventional data, effectively reducing sample complexity when trying to model the interventional distribution. This can be beneficial for causal effect estimation when we assume that we only have access to a small number of interventional samples and a large number of observational samples.

\section{Estimation in the nonlinear, non-Gaussian case}
\label{sec:implementations}

In the final application of our causal reduction, we derive a flexible parameterization of the reduced model, which enables us to estimate the observational and interventional distributions by jointly learning from observational and interventional data without making strong parametric assumptions, where we parameterize the functions $F$ and $G$ in Equations~\ref{eq:reduced_se_X} and \ref{eq:shared_function} to learn the model from data.
Intuitively, the reduced SCM derived in Section \ref{sec:add_latents} tells us that the parameters of $G$ are shared among observational and interventional samples for an intervention on $\mathbf{X}$, whereas the parameters of $F$ are not. For the remaining part of this section, we focus on one-dimensional treatment outcome pairs, i.e.\ $X \in \mathcal{X}= \mathbb{R}$ and $Y \in \mathcal{Y}= \mathbb{R}$, and (optionally) an $L$-dimensional observed confounder $\mathbf{C} \in \mathcal{C} = \mathbb{R}^L$. We now use the following bijective transformations between observed variables $x, y$ and latent variables $u, v$. We define $u = f_{\boldsymbol{\phi}}(x)$ and
$v = g_{x, u; \boldsymbol{\theta}}(y)$
where the functions $f_{\boldsymbol{\phi}}$ and $g_{x,u;\boldsymbol{\theta}}$ are invertible 
for all $\boldsymbol{\phi}, x, u, \boldsymbol{\theta}$. Here, $f_{\boldsymbol{\phi}} = F^{-1}$ from Equation~\ref{eq:reduced_se_X}
and $g_{x, u; \boldsymbol{\theta}}$ is the inverse of $v \mapsto G(u,v,x)$ from Equation~\ref{eq:shared_function} (for fixed $u,x,\boldsymbol{\theta}$).\footnote{In the presence of observed confounding, we simply have to replace the functions $f_{\boldsymbol{\phi}}$ and $g$ by functions $f_{\mathbf{c};\boldsymbol{\phi}}$ and $g_{x, u, \mathbf{c}; \boldsymbol{\theta}}$.}

The SCM also specifies that $u \sim \mathcal{N}(0,1)$, $v \sim \mathcal{N}(0,1)$ and $u \Indep v$.
%Without loss of generality we assume independent, standard Gaussian distributions for $u, v$: $p_U(u) = \mathcal{N}(0, 1) \Indep p_V(v) = \mathcal{N}(0, 1)$.
The transformations defined above allow us to rewrite the joint likelihood using the change of variable formula
\begin{align}
    \log p(x,y) = \log p_V(g_{x,f_{\boldsymbol{\phi}}(x); \boldsymbol{\theta}}(y)) + \log \left\lvert \frac{\partial g_{x, f_{\boldsymbol{\phi}}(x); \boldsymbol{\theta}}(y)}{\partial y} \right\lvert
    + \log p_U(f_{\boldsymbol{\phi}}(x)) + \log \left\lvert \frac{\partial f_{\boldsymbol{\phi}}(x)}{\partial x} \right\lvert,
\label{eq:log_p_obser_final}
\end{align}
where we substituted $u = f_{\boldsymbol{\phi}}(x)$ into $g_{x, u; \boldsymbol{\theta}}(y)$. 
The parameters $\boldsymbol{\phi}$ and $\boldsymbol{\theta}$ are jointly updated by minimizing $\sum_{o=1}^{N_O}- \log p(x^O_o,y^O_o)$ given $N_O$ observational training samples $(x^O_1,y^O_1),\dots,(x^O_{N_O},y^O_{N_O})$.

\label{sec:implementations_inter}
In contrast to the observational setting, we only have to consider the conditional likelihood $p(y|\doit(x))$ in the interventional case. Since we cannot use $f_{\boldsymbol{\phi}}(x)$ to impute $u$, we instead marginalize over $u$
\begin{align}
     \log p(y\mid\doit(x)) = \log \int p_V(g_{x,u; \boldsymbol{\theta}}(y)) \left\lvert \frac{\partial g_{x, u; \boldsymbol{\theta}}(y)}{\partial y} \right\lvert p(u) du.
\label{eq:log_p_inter_final}
\end{align}
Since this is a one-dimensional integral, we can approximate it accurately numerically by means of the 
trapezoidal rule. The parameter $\boldsymbol{\theta}$ can be updated by minimizing $\sum_{i=1}^{N_I}- \log p(y^I_i|\doit(x^I_i))$ given $N_I$ interventional training samples $(x^I_1,y^I_1),\dots,(x^I_{N_I},y^I_{N_I})$.

Assuming we have $N_O$ observational samples and $N_I$ interventional samples, we define the full loss as given by
\begin{align}
  \ell = \frac{1}{N_O} \sum_{o=1}^{N_O} -\log p(x^O_o, y^O_o) + \frac{1}{N_I} \sum_{i=1}^{N_I}
    - \log p(y^I_i|\doit(x^I_i)).
    \label{eq:final_loss}
\end{align}
The parameters $\boldsymbol{\phi}$ and $\boldsymbol{\theta}$ of the transformation $f$ and $g$ are learned by minimizing the loss using gradient descent. In Equation \ref{eq:final_loss}, we scale each loss term by the number of samples used for training to balance their contribution during optimization.
\label{sec:experiment}
We perform a series of experiments on simulated data, where the causal relationships between all variables are nonlinear, showing that we can significantly reduce the number of interventional samples required to estimate the interventional distribution $p(y|\doit(x))$ by training jointly with (possibly confounded) observational and interventional samples. Throughout this section, we are using the parameterization described in Section \ref{sec:implementations}, where we use linear rational spline flows \citep{dolatabadi_invertible_2020}. For a detailed description of this choice, see Appendix \ref{sec:flows}. We perform two sets of experiments: (1) We consider $K$ latent confounders $Z_1, \dots, Z_K \in \mathbb{R}$ with an arbitrary dependency structure. (2) We consider $L$ additional, observed confounders $C_1, \dots, C_L \in \mathbb{R}$ with an arbitrary dependency structure. All flow models are implemented with the automatic differentiation packages Pytorch \citep{paszke_pytorch_nodate} and Pyro \citep{bingham_pyro_2018}. All code is available under \url{https://github.com/max-ilse/CausalReduction}.

\paragraph{Without observed confounders}
\label{sec:nonlinear_experiment}
We simulate cause and effect pairs from the SCM with structural equations: $X=F(E_X, \mathbf{Z})$, $Y=G(X, E_Y, \mathbf{Z})$. A single dataset consists of observational and interventional samples. All causal relationships are simulated using fully connected neural networks with a single hidden layer, where the weights are randomly initialized. The activation functions are REctified Linear Units (ReLUs). As a result, the simulated causal mechanisms are nonlinear.  The values of $E_X, E_Y, \mathbf{Z}$ and $\doit(X)$ are sampled from a random distribution, as seen in \cite{mooij_distinguishing_2016}. A detailed step-by-step description of the simulation procedure is given in Appendix \ref{sec:data_generation_details}.
\begin{table}[h]
    \caption[Comparison of a flow model trained with interventional samples only and a flow model trained with interventional and observational samples.]{Comparison of a flow model trained with interventional samples only and a flow model trained with interventional and observational samples. We calculate the ratio $N^{*}_I/N_I$, where $N^{*}_I$ is the number of interventional samples necessary to match the interventional test log-likelihood of a flow model trained with $N_I$ interventional and $N_O=1000$ observational samples. E.g.\ in the case of dataset 3 and $N_I=100$, if we were to use only interventional samples, we would require twice as many interventional samples compared to using 100 interventional and $1000$ observational samples. For datasets 11 to 15, we simulate an additional observed confounder $\mathbf{C}$. Note that if a large number of interventional samples ($ 250 < N_I \leq 1000$) are available the improvements become smaller as shown in Appendix \ref{sec:all_results}.}
\centering
\begin{tabular}{ c | c c c c c c c c c c | c c c c c}
 $N_I$ & 1 & 2 & 3 & 4 & 5 & 6 & 7 & 8 & 9 & 10 & 11 & 12 & 13 & 14 & 15\\
 \hline
  50 & 1.4 &1.8 &2.2 &1.2 &0.2 &2.2 &2.1 &1.7 &1.9 &1.6 &3.2 &3.2 &2.2 &2.7 &3.2\\
  100 & 0.8 & 2.6 &2.0 &1.5 &0.3 &2.1 &2.0 &2.5 &2.0 &2.1 &3.2 &2.9 &2.5 &3.0 &2.5\\
  250 &1.0  &1.5  &1.8 &1.6 &0.5 &1.7 &1.1 &1.5 &1.2 &1.7 &2.4 &2.3 &2.3 &2.1 &1.7\\
\end{tabular}
    \label{tab:non_linear_results}
\end{table}
\begin{figure}[h]
\centering
\includegraphics[scale=0.5]{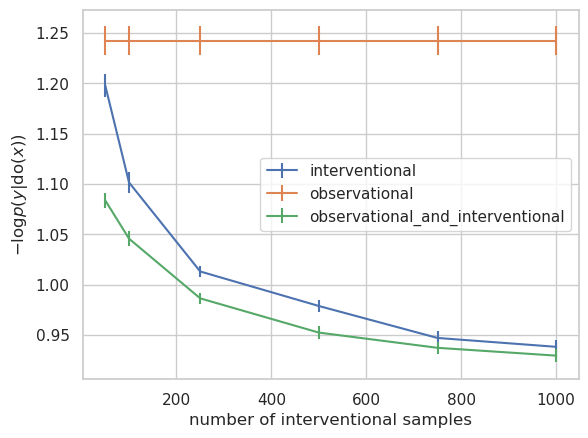}
\includegraphics[scale=0.5]{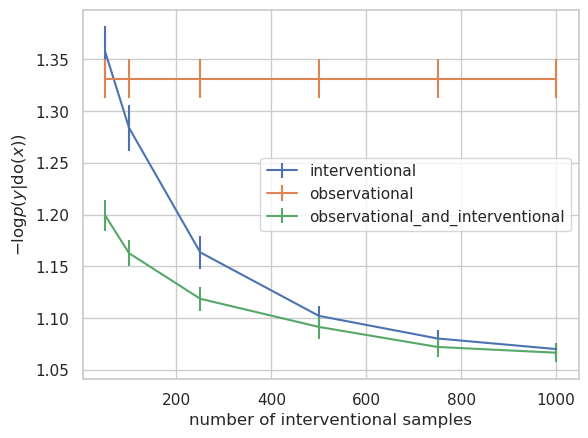}
  \caption{Example comparison for dataset 3 and 15 of a flow model trained with 1000 observational samples, a flow model trained with {50, 100, 250, 500, 750, 1000} interventional samples, and a flow model trained with both 1000 observational samples and {50, 100, 250, 500, 750, 1000} interventional samples. All flow models are evaluated on 1000 interventional samples from the test set.}
\label{fig:example_results}
\end{figure}
In this experiment we are interested in estimating the interventional distribution $p(y|\doit(x))$. For each dataset, we train three variants of our reduced causal model parameterized with normalizing flows. The first flow model is trained using only observational data, see Equation \ref{eq:log_p_obser_final}. The second flow model is trained using only interventional data, see Equation \ref{eq:log_p_inter_final}. The third flow model is trained using observational and interventional data jointly, see Equation \ref{eq:final_loss}. For each of the ten datasets, we keep the number of observational samples constant at 1000 and use an increasing number of interventional samples 50, 100, 250, 500, 750, 1000, resulting in six experiments per dataset. For example, in the case of 50 interventional and 1000 observational samples, the first flow model is trained with 1000 observational samples, the second flow model is trained with 50 interventional samples, and the third flow model is jointly trained with 1000 observational and 50 interventional samples, see Figure \ref{fig:example_results} for an example.

In Appendix \ref{sec:all_results}, we provide extensive visualizations of the results of all experiments, including scatter plots of training data, samples from the trained flow models, negative log-likelihood values for all flow models on the interventional and observational test sets. To summarize our findings, we calculate the ratio of samples required to reach the same performance, measured in averaged negative log-likelihood when only using interventional samples. 

In Table \ref{tab:non_linear_results} we see that we can substantially reduce the number of interventional samples required when using an additional 1000 observational samples in eight of the ten datasets. The results in Table \ref{tab:non_linear_results}, dataset 1 to 10, are in agreement with qualitative results in Appendix \ref{sec:all_results}, where we find that samples from the flow model trained with interventional and observational data better resemble the training data compared to samples from a flow model trained with interventional data only.

\paragraph{With observed confounders}
\label{sec:nonlinear_experiment_with_confounder}
We now consider the case of an additional $L$-dimensional observed confounder $\mathbf{C}$. We use the same setup as in Section \ref{sec:nonlinear_experiment} to simulate triples $(x, y, \mathbf{c})$. We use the following nonlinear causal mechanisms to generate treatment $X$ and outcome $Y$: $X = f(E_X, \mathbf{Z, C})$ and $Y = g(X, E_Y, \mathbf{Z,C})$, a detailed description of the simulation procedure is given in Appendix \ref{sec:data_generation_details}. For each of the five datasets, we keep the number of observational samples constant at 1000 and use an increasing number of interventional samples: 50, 100, 250, 500, 750, 1000, resulting in six experiments per dataset. We compare three flow models trained with observational, interventional, and observational plus interventional data, respectively. The training details are the same as in Section \ref{sec:nonlinear_experiment}. An extensive comparison of the three flow models, as well as visualizations for each dataset, can be found in Appendix \ref{sec:all_results_with_confounder}. The main result of the experiments with additional, observed confounders is the following: For each of the five datasets, we can substantially reduce the required number of interventional samples with our flow model trained with observational and interventional data, see Table \ref{tab:non_linear_results}, dataset 11 to 15. We find that we can reduce the number of required samples by a factor of two to three when training with 1000 additional observational samples.

\section{Conclusion}
We showed that without loss of generality when modeling an unobserved confounder of a treatment and an outcome variable, we may assume that this confounder takes values in the same space as the treatment variable.
Applying this insight to the setting of discrete treatment and outcome, we can easily derive bounds between observational and interventional distributions that can be exploited for estimation purposes.
We pointed out that in certain cases with highly unbalanced observational samples, the accuracy of the causal effect estimate can be improved by incorporating observational data.
Furthermore, in the linear-Gaussian setting, we derived equality constraints between the parameters of the observational and interventional distributions, showing that these distributions are not independent. 
Finally, for the general continuous (possibly nonlinear and non-Gaussian) setting, we proposed a flexible parameterization of the reduced causal model using normalizing flows. 
This parameterization allows training of a single flow model from combined observational and interventional data. 
In simulations, for 13 out of 15 simulated datasets, we could substantially reduce the number of interventional samples if sufficient observational samples are available without sacrificing accuracy. 

Together, our results suggest that there is still untapped potential to obtain more accurate estimates of causal effects by combining observational and interventional data while allowing for latent confounding in the observational regime.
Our work opens up practical applications and further theoretical questions regarding the precise nature of the relationship between observational and interventional distributions in parametric and non-parametric settings. 

Possible future work includes (i) investigating the potential for improved causal effect estimation from combined data in clinical applications, (ii) extending the reduction operation to more than two observed variables, and (iii) applying the flow model to high-dimensional outcome variables, e.g., medical images.

\bibliography{references}
\newpage
\appendix
\section{Causal Reduction theorem}
\label{sec:causal_red_thm}
\begin{theoremEnd}{Thm}[Causal Reduction]\label{thm:causal_reduction}
  Let $\mathcal{M}$ be a causal Bayesian network with observed variables $\mathbf{X} \in \mathcal{X},\mathbf{Y} \in \mathcal{Y}$ and latent variables$,Z_1 \in \mathcal{Z}_1,\dots, Z_K \in \mathcal{Z}_K$ such that $\mathbf{Y}$ is not an ancestor of $\mathbf{X}$.

  Then there exists a causal Bayesian network $\mathcal{M}^*$ with observed variables $\mathbf{X} \in \mathcal{X}$ and $\mathbf{Y}\in\mathcal{Y}$ and a single latent confounder $\mathbf{W}\in\mathcal{X}$ (that takes values in the same space as $\mathbf{X}$) such that $\mathcal{M}^*$ is interventionally equivalent to $\mathcal{M}$ with respect to perfect interventions on the observed variables $\mathbf{X}$ and $\mathbf{Y}$:
  $$p_{\mathcal{M}}(\mathbf{x},\mathbf{y}) = p_{\mathcal{M}^*}(\mathbf{x},\mathbf{y})$$
  $$p_{\mathcal{M}}(\mathbf{x} \mid \doit(\mathbf{y})) = p_{\mathcal{M}^*}(\mathbf{x} \mid \doit(\mathbf{y}))$$
  $$p_{\mathcal{M}}(\mathbf{y} \mid \doit(\mathbf{x})) = p_{\mathcal{M}^*}(\mathbf{y} \mid \doit(\mathbf{x})).$$
\end{theoremEnd}
\begin{proofEnd}
While the main text preceding Theorem~\ref{thm:causal_reduction} already provides a proof, we show here
explicitly that the reduction operation commutes with a perfect intervention on $\mathbf{X}$.

\begin{equation*}\begin{split}
  p(\mathbf{y}|\doit(\mathbf{x})) 
  & = \int_{\mathcal{Z}_1} \dots \int_{\mathcal{Z}_K} p(\mathbf{y}, z_1, \dots, z_K|\doit(\mathbf{x})) dz_1 \dots dz_K \\
  & \stackrel{(a)}{=} \int_\mathcal{Z} p(\mathbf{y,z}|\doit(\mathbf{x})) d\mathbf{z} \\
  & \stackrel{(b)}{=} \int_\mathcal{Z} p(\mathbf{y}|\mathbf{x}, \mathbf{z})p(\mathbf{z}) d\mathbf{z} ,\\
  & \stackrel{(c)}{=} \int_\mathcal{Z} \left( \int_\mathcal{X} p(\mathbf{y}|\mathbf{x}, \mathbf{z})p(\mathbf{z}) p(\mathbf{w}|\mathbf{z}) d\mathbf{w} \right) d\mathbf{z},\\
  & \stackrel{(d,e)}{=} \int_\mathcal{X} \left( \int_\mathcal{Z} p(\mathbf{y}|\mathbf{x}, \mathbf{z})p(\mathbf{w}) p(\mathbf{z}|\mathbf{w}) d\mathbf{z} \right) d\mathbf{w}, \\
  & \stackrel{(\ref{eq:red-7})}{=} \int_\mathcal{X} p(\mathbf{y}|\mathbf{x}, \mathbf{w})p(\mathbf{w})d\mathbf{w}\,.
\end{split}\end{equation*}
\end{proofEnd}
We call the causal Bayesian network $\mathcal{M}^*$ a \emph{causal reduction} of $\mathcal{M}$ since it will typically be the case that the latent space will be reduced, while the causal semantics are preserved by construction. The single latent confounder $\mathbf{Z}$ in $\mathcal{M}^*$ will then more parsimoniously represent the causal influence of \emph{all} latent confounders of $\mathbf{X}$ and $\mathbf{Y}$ in $\mathcal{M}$. 

From this result, it also follows that a reduced causal Bayesian network can be constructed from a given acyclic structural causal model.
Indeed, this is immediate since any acyclic structural causal model can also be interpreted as a causal Bayesian network with latent variables. 
Extending the derivation to \emph{simple} structural causal models (a convenient class of structural causal models that can represent causal cycles, such as feedback loops \citep{Bongers++_AOS_21}) is straightforward, as long as $\mathbf{X}$ and $\mathbf{Y}$ are not part of a causal cycle (although the other variables might be involved in cycles).
Even more generally, we can also start with given potential outcomes and obtain a reduced causal Bayesian network. 
This is shown explicitly in Appendix~\ref{app:potential_outcomes}.

\section{Reduction with observed confounders}
\label{sec:reduction_with_confounders}

There are many scenarios where we are interested in estimating the conditional causal effect of interventions given additional covariates $\mathbf{C}$ that might confound treatment and outcome, for example when estimating the efficacy of a vaccine depending on age. We again consider a treatment variable $\mathbf{X} \in \mathcal{X}$, an outcome variable $\mathbf{Y} \in \mathcal{Y}$, and a set of $K$ latent confounders $Z_1,\dots,Z_K$ in arbitrary standard measurable spaces $\mathcal{Z}_1,\dots,\mathcal{Z}_K$. In addition, let there be $L$ observed confounders $C_1, \dots, C_L$ of $\mathbf{X}$ and $\mathbf{Y}$, again in arbitrary standard measurable spaces. We allow for arbitrary causal relations and dependencies between the confounders. In the following, we summarize all observed confounders using a single variable $\mathbf{C} = (C_1,\dots,C_L) \in \mathcal{C}$ and all latent confounders as $\mathbf{Z} = (Z_1,\dots,Z_K) \in \mathcal{Z}$. We follow a similar sequence of steps as in Section \ref{sec:reduction} to derive a reduced causal model of the following form
\begin{align}
    \label{eq:reduction_with_confounder}
    p(\mathbf{x,y,w, c}) = p(\mathbf{y}|\mathbf{x},\mathbf{w}, \mathbf{c})  p(\mathbf{x|w})  p(\mathbf{w|c})  p(\mathbf{c}).
\end{align}
as illustrated in Figure~\ref{fig:graphical_explanation_with_confounders} (a--d).
At every step, the Bayesian network is observationally equivalent to the ones before and also interventionally equivalent for perfect interventions on any subset of $\{\mathbf{X},\mathbf{Y}\}$.

Specializing now to continuous treatment $\mathbf{X} \in \mathcal{X} = \mathbb{R}^M$ and outcome $\mathbf{Y} \in \mathcal{Y} = \mathbb{R}^N$,
we can use a similar approach as in Section \ref{sec:add_latents}, and in addition marginalize out $\mathbf{W}$ as seen in Figure \ref{fig:graphical_explanation_with_confounders} (g), 
to convert the causal Bayesian Network into an SCM with structural equations of the form given below
\begin{align}
  \mathbf{X} & = F(\mathbf{U}, \mathbf{C}), \label{eq:scm_X}\\
    \mathbf{Y} & = G(\mathbf{V}, \mathbf{X},\mathbf{U}, \mathbf{C}),  \label{eq:scm_Y}
\end{align}
where $\mathbf{U} \sim \mathcal{N}(\mathbf{0, I}_M)$, $\mathbf{V} \sim \mathcal{N}\left(\mathbf{0, I}_N\right)$, $\mathbf{U} \Indep \mathbf{V}$ and $F$ and $G$ are two deterministic maps. 
This is illustrated in Figure~\ref{fig:graphical_explanation_with_confounders} (e--h).
Again, at every step, the Bayesian network is observationally equivalent to the ones before and also interventionally equivalent for perfect interventions on any subset of $\{\mathbf{X},\mathbf{Y}\}$.

\begin{figure}[h]
\Large
\begin{subfigure}% {.45\textwidth}
\centering
\scalebox{0.6}{\begin{tikzpicture}
\node [node](z2) {$Z_2$};
\node [node, left = of z2](z1) {$Z_1$};
\node [nodeobserved, right = of z2](zd) {$C_1$};
\node [text width=3cm, left = of zd, xshift=3.4cm](dots) {$\mathbf{\dots}$};
\node [node, below right = of z2, xshift=-0.65cm](z3) {$Z_K$};
\node [nodeobserved, below left = of z2, xshift=0.65cm](z4) {$C_L$};
\node [nodeobserved, below = of z3](x) {$\mathbf{Y}$};
\node [nodeobserved, below = of z4](y) {$\mathbf{X}$};

\draw [arrow] (z1) -- (z2);
\draw [arrow] (z1) -- (z3);
\draw [arrow] (z2) -- (z3);
\draw [arrow] (z3) -- (z4);
\draw [arrow] (z1) -- (z4);
\draw [arrow] (zd) -- (z4);
\draw [arrow] (zd) -- (z3);
\draw [arrow] (z4) -- (y);
\draw [arrow] (z4) -- (x);
\draw [arrow] (z3) -- (x);
\draw [arrow] (y) -- (x);
\draw [arrow] (zd) to [out=-90,in=45] (x);
\end{tikzpicture}}
% \subcaption{}
\end{subfigure}
\begin{subfigure}% {.45\textwidth}
\centering
\scalebox{0.6}{\begin{tikzpicture}
\node [node](z) {$\mathbf{Z}$};
\node [nodeobserved, right = of z](c) {$\mathbf{C}$};
\node [nodeobserved, below = of z](x) {$\mathbf{X}$};
\node [nodeobserved, below = of c](y) {$\mathbf{Y}$};

\draw [arrow] (c) -- (z);
\draw [arrow] (c) -- (x);
\draw [arrow] (c) -- (y);
\draw [arrow] (z) -- (y);
\draw [arrow] (z) -- (x);
\draw [arrow] (x) -- (y);
\end{tikzpicture}}
% \subcaption{}
\end{subfigure}
\begin{subfigure}% {.45\textwidth}
\centering
\scalebox{0.6}{\begin{tikzpicture}
\node [node](z) {$\mathbf{Z}$};
\node [node, left = of z](w) {$\mathbf{W}$};
\node [nodeobserved, right = of z](c) {$\mathbf{C}$};
\node [nodeobserved, double, below left = of z, yshift=-0.25cm, xshift=0.65cm](x) {$\mathbf{X}$};
\node [nodeobserved, below right = of z, yshift=-0.25cm, xshift=-0.65cm](y) {$\mathbf{Y}$};

\draw [arrow] (c) -- (z);
\draw [arrow] (c) -- (y);
\draw [arrow] (z) -- (y);
\draw [arrow] (z) -- (w);
\draw [arrow] (w) -- (x);
\draw [arrow] (x) -- (y);
\draw [arrow] (c) to [out=150,in=30] (w);
\end{tikzpicture}}
% \subcaption{}
\end{subfigure}
\begin{subfigure}% {.45\textwidth}
\centering
\scalebox{0.6}{\begin{tikzpicture}
\node [node](z) {$\mathbf{Z}$};
\node [node, left = of z](w) {$\mathbf{W}$};
\node [nodeobserved, right = of z](c) {$\mathbf{C}$};
\node [nodeobserved, double, below left = of z, yshift=-0.25cm, xshift=0.65cm](x) {$\mathbf{X}$};
\node [nodeobserved, below right = of z, yshift=-0.25cm, xshift=-0.65cm](y) {$\mathbf{Y}$};

\draw [arrow] (c) -- (z);
\draw [arrow] (c) -- (y);
\draw [arrow] (z) -- (y);
\draw [arrow] (w) -- (z);
\draw [arrow] (w) -- (x);
\draw [arrow] (x) -- (y);
\draw [arrow] (c) to [out=150,in=30] (w);
\end{tikzpicture}}
% \subcaption{}
\end{subfigure}
\begin{subfigure}% {.45\textwidth}
\centering
\scalebox{0.6}{\begin{tikzpicture}
\node [node](w) {$\mathbf{W}$};
\node [nodeobserved, right = of w](c) {$\mathbf{C}$};
\node [nodeobserved, double, below = of w](x) {$\mathbf{X}$};
\node [nodeobserved, below = of c](y) {$\mathbf{Y}$};

\draw [arrow] (c) -- (w);
\draw [arrow] (c) -- (y);
\draw [arrow] (w) -- (y);
\draw [arrow] (w) -- (x);
\draw [arrow] (x) -- (y);
\end{tikzpicture}}
% \subcaption{}
\end{subfigure}
\begin{subfigure}% {.45\textwidth}
\centering
\scalebox{0.6}{\begin{tikzpicture}
\node [node,double](w) {$\mathbf{W}$};
\node [node,left = of w](u) {$U$};
\node [nodeobserved, right = of w](c) {$\mathbf{C}$};
\node [nodeobserved, double, below = of w](x) {$\mathbf{X}$};
\node [nodeobserved, below = of c](y) {$\mathbf{Y}$};

\draw [arrow] (u) -- (w);
\draw [arrow] (c) -- (w);
\draw [arrow] (c) -- (y);
\draw [arrow] (w) -- (y);
\draw [arrow] (w) -- (x);
\draw [arrow] (x) -- (y);
\end{tikzpicture}}
% \subcaption{}
\end{subfigure}
\begin{center}
\begin{subfigure}% {.45\textwidth}
\centering
\scalebox{0.6}{\begin{tikzpicture}
\node [node](u) {$U$};
\node [nodeobserved, right = of u](c) {$\mathbf{C}$};
\node [nodeobserved, double, below = of u](x) {$\mathbf{X}$};
\node [nodeobserved, below = of c](y) {$\mathbf{Y}$};

\draw [arrow] (u) -- (x);
\draw [arrow] (u) -- (y);
\draw [arrow] (c) -- (x);
\draw [arrow] (c) -- (y);
\draw [arrow] (x) -- (y);
\end{tikzpicture}}
% \subcaption{}
\end{subfigure}
\begin{subfigure}% {.45\textwidth}
\centering
\scalebox{0.6}{\begin{tikzpicture}
\node [node](u) {$U$};
\node [nodeobserved, right = of u](c) {$\mathbf{C}$};
\node [node,right = of c](v) {$V$};
\node [nodeobserved, double, below = of u](x) {$\mathbf{X}$};
\node [nodeobserved, double, below = of c](y) {$\mathbf{Y}$};

\draw [arrow] (u) -- (y);
\draw [arrow] (u) -- (x);
\draw [arrow] (c) -- (x);
\draw [arrow] (c) -- (y);
\draw [arrow] (x) -- (y);
\draw [arrow] (v) -- (y);
\end{tikzpicture}}
% \subcaption{}
\end{subfigure}
\end{center}
    \caption[A graphical explanation of our reduction technique in the presence of both observed and latent confounders.]{A graphical explanation of our reduction technique in the presence of both observed and latent confounders. (a) We assume a treatment variable $\mathbf{X}\in\mathcal{X}$, an outcome variable $\mathbf{Y}\in\mathcal{Y}$, latent confounders $Z_1,\dots,Z_K$, and observed confounders $C_1,\dots,C_L$, with arbitrary causal and probabilistic relations between the confounders. (b) We combine the latent confounders into $\mathbf{Z} \in \mathcal{Z}$ and the observed confounders into $\mathbf{C} \in \mathcal{C}$, and factorize their joint distribution as $p(\mathbf{z} \mid \mathbf{c}) p(\mathbf{c})$. (c) We create a copy of $\mathbf{X}$ called $\mathbf{W}$. (d) We refactorize $p(\mathbf{w,z,c})$ as $p(\mathbf{z}\mid\mathbf{w},\mathbf{c})p(\mathbf{w}\mid \mathbf{c})p(\mathbf{c})$. (e) We marginalize over $\mathbf{Z}$. (f) We reparameterize $p(\mathbf{w}\mid \mathbf{c})$ using Theorem \ref{gaus_plus_map} as a deterministic function, introducing an independent noise variable $\mathbf{U}$. (g) We marginalize over $\mathbf{W}$. (h) We reparameterize $p(\mathbf{y} \mid \mathbf{x},\mathbf{u},\mathbf{c})$ with Theorem \ref{gaus_plus_map} as a deterministic function, introducing an independent noise variable $\mathbf{V}$. Note that at every step (a--h) the models remain observationally equivalent (i.e., $p(\mathbf{c,x,y})$ is invariant), and also interventionally equivalent with respect to perfect interventions on $\mathbf{X}$ and $\mathbf{Y}$ (i.e., $p(\mathbf{y},\mathbf{c}\mid\doit(\mathbf{x}))$ and $p(\mathbf{x},\mathbf{c}\mid\doit(\mathbf{y}))$ are invariant).}
    \label{fig:graphical_explanation_with_confounders}
\end{figure}
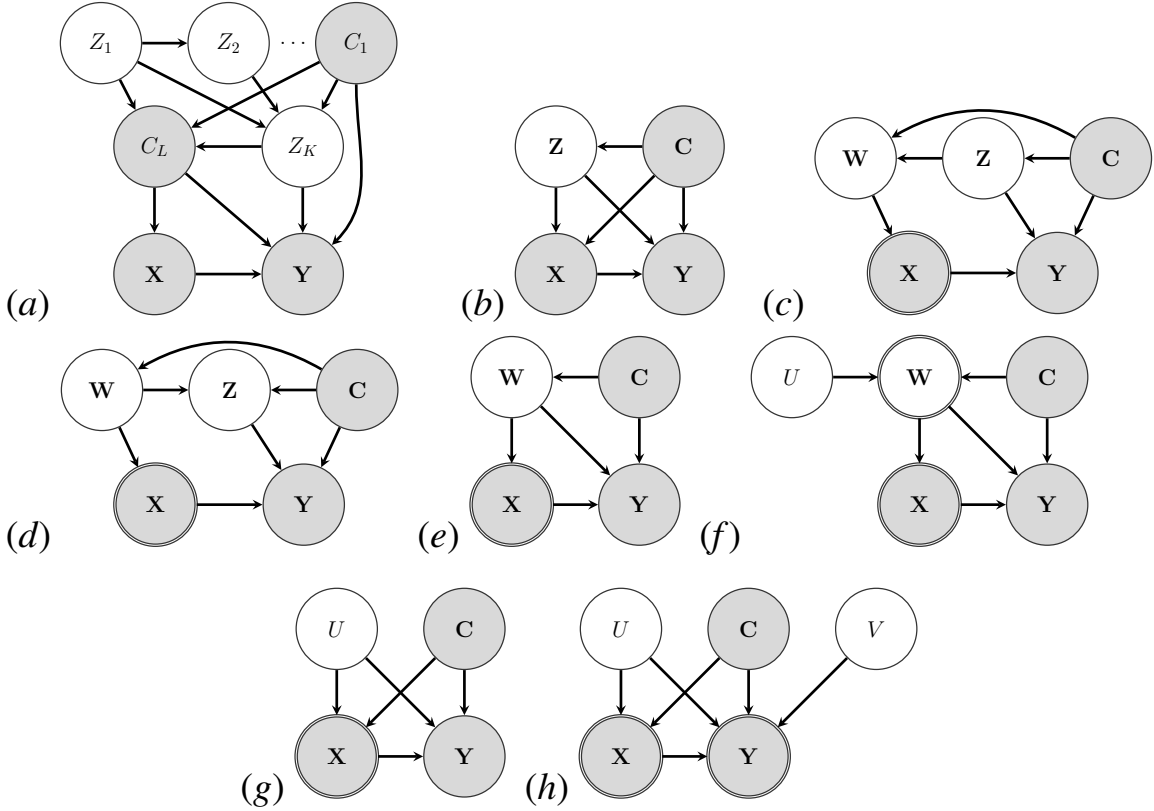

\section{Bounds on the causal effect}
\label{sec:manski}
Let us consider treatment variable $X \in \C{X}$, outcome variable $Y \in \C{Y}$, observed confounder $\mathbf{C} \in \C{C}$ and our surrogate latent confounder $W \in \C{X}$ that was obtained by the causal reduction operation. 
Assume that $\C{X}$, $\C{Y}$ and $\C{C}$ are finite or countably infinite.
For simplicity, we start with the case without observed covariates $\B{C}$.
%We will use shorthand notation for probability mass functions, for example, $p(x,y) = p(X=x,Y=y)$.
In the reduced causal model, we obtain the following expression for the causal effect of $X$ on $Y$ (i.e., the interventional distributions) from Equation~\ref{eq:red-6}:
\begin{equation}\label{eq:p_y_do_x_reduced}
  p(y \mid \doit(x)) = \sum_{w \in \C{X}} p(w) p(y \mid x,w),
\end{equation}
while for the observational distribution we obtain:
\begin{equation}\label{eq:p_x_y_reduced}\begin{split}
  p(x, y) &= \sum_{w \in \C{X}} p(w) p(x \mid w) p(y \mid x,w) \\
  &= \sum_{w \in \C{X}} p(w)  \delta_{w}(x) p(y \mid x,w) \\
  &= p(W=x) p(y \mid X=x,W=x).
\end{split}\end{equation}

From this we immediately obtain bounds relating the observational and interventional distributions. 
Splitting the sum in \eref{eq:p_y_do_x_reduced} and substituting \eref{eq:p_x_y_reduced}, we get:
\begin{equation*}\begin{split}
  p(y \mid \doit(x)) 
  & = \sum_{w \in \C{X}} p(w) p(y \mid x,w) \\
  & = p(W=x) p(y \mid X=x,W=x) + \sum_{\substack{w \in \C{X}\\ w \ne x}} p(w) p(y \mid x,w) \\
  & = p(x,y) + \sum_{\substack{w \in \C{X}\\ w \ne x}} p(w) p(y \mid x,w).
\end{split}\end{equation*}
Since $0 \le p(w) \le 1$ and $0 \le p(y \mid x,w) \le 1$, we can bound
  $$0 \le \sum_{\substack{w \in \C{X}\\ w \ne x}} p(w) p(y \mid x,w) \le \sum_{\substack{w \in \C{X}\\ w \ne x}} p(w) = 1 - p(W=x) = 1 - p(x),$$
and we conclude
  \begin{equation}\label{eq:discrete_causal_effect_bound}
    p(x,y) \le p(y \mid \doit(x)) \le p(x,y) + 1 - p(x)
  \end{equation}
for all $x \in \C{X}$, $y \in \C{Y}$.
This means that the interventional distributions corresponding to the causal effect of $X$ on $Y$ are bounded by the observational distribution of $X$ and $Y$.
For binary treatment and outcome, this bound is not novel, as it was already derived in \cite{ManskiNagin:98} % equations (8) and (9) in this reference
and can also be obtained from the instrumental inequality of \cite{balke_bounds_1997} under the assumption of perfect compliance.
%\footnote{These bounds are not novel: they can also be derived as special cases of the instrumental inequality of Balke \& Pearl, corresponding to the case that compliance is perfect. Indeed, we can use the following correspondence to turn our observational/interventional multi-task setting into a purely observational JCI-setting that matches Pearl's instrumental setup. We get the following correspondences:
%$$p_{yx.1} = p(y,x|z=1) := p(y\mid\doit(x)) p(\doit(x))$$
%$$p_{yx.0} = p(y,x|z=0) := p(y,x)$$
%between Pearl's notation $p_{yx.z}$ and our setting. Then, the bounds (8.15) in Pearl's book become:
%$$p(y=1\mid\doit(x=0)) \ge p_{10.1} = p(y=1\mid\doit(x=0)) p(\doit(x=0))$$
%which is trivial, and
%$$p(y=1\mid\doit(x=0)) \ge p_{10.0} = p(y=1,x=0)$$
%which corresponds with our bound, and 2 more bounds that I cannot directly relate to our work.
%However, those bounds involve the quantity $p(\doit(x))$ which is not present in our setting, so it could be that the bound we found is complete for our setting.
%}
%\footnote{Manski (1990) "Nonparametric bounds on treatment effects" derives the following bound in the binary case:
%p(y=1|do x=1) p(x=1) - p(y=1|do x=0) p(x=0) - p(x=1)
%\le p(y=1|do x=1) - p(y=0|do x=1) - p(y=1|do x=0) + p(y=1|do x=0) (=treatment effect)
%\le p(x=0) + p(y=1|do(x=1)) p(x=1) - p(y=1|do(x=0) p(x=0)
%He also has a bound for the discrete case, and derives several other bounds.}

In the presence of observed confounders $\B{C}$, one similarly derives the bound
  \begin{equation}\label{eq:discrete_causal_effect_bound_conditional}
    p(x,y|\B{c}) \le p(y \mid \doit(x),\B{c}) \le p(x,y|\B{c}) + 1 - p(x|\B{c})
  \end{equation}
for all $x \in \C{X}$, $y \in \C{Y}$, $\B{c} \in \C{C}$.

\section{Figures of bounds}
\label{sec:fig_bounds}
Illustration of the bounds in Equation~\ref{eq:bound_binary_x0} and Equation~\ref{eq:bound_binary_x1} are shown in Figure \ref{fig:bounds}.
\begin{figure}[h]
  \includegraphics[width=0.48\textwidth]{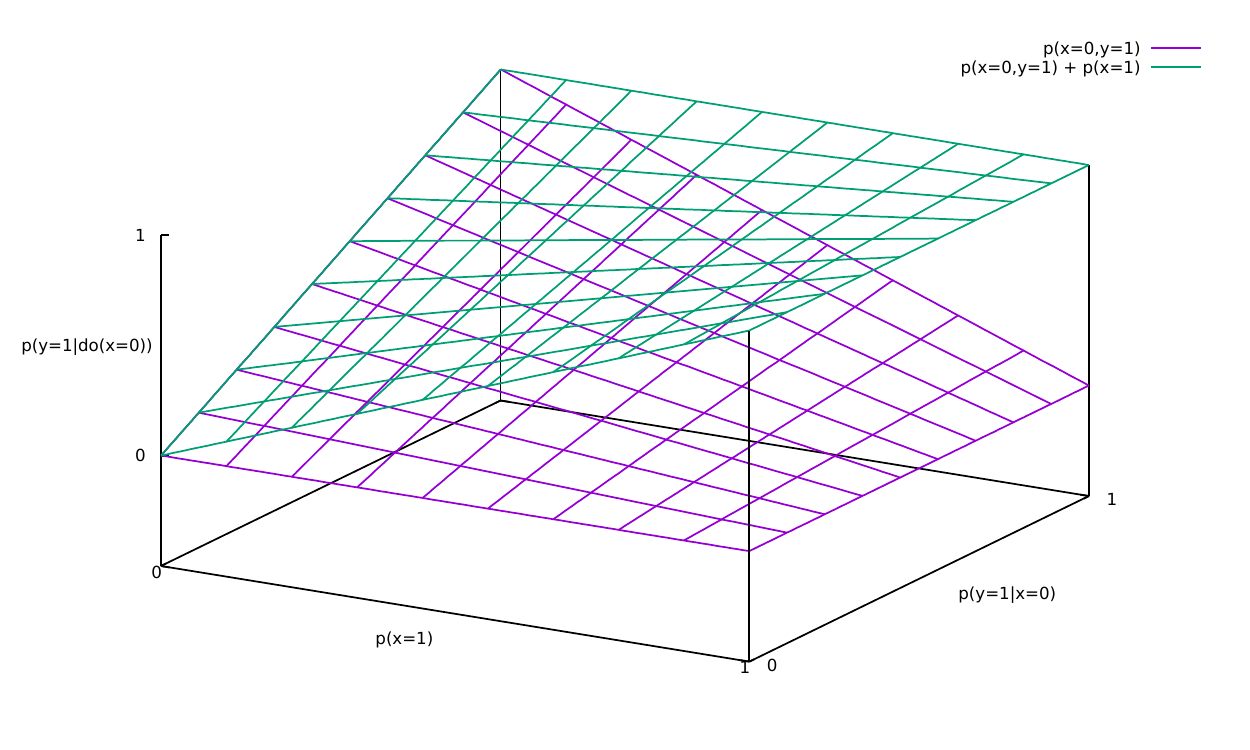}\hfill\includegraphics[width=0.48\textwidth]{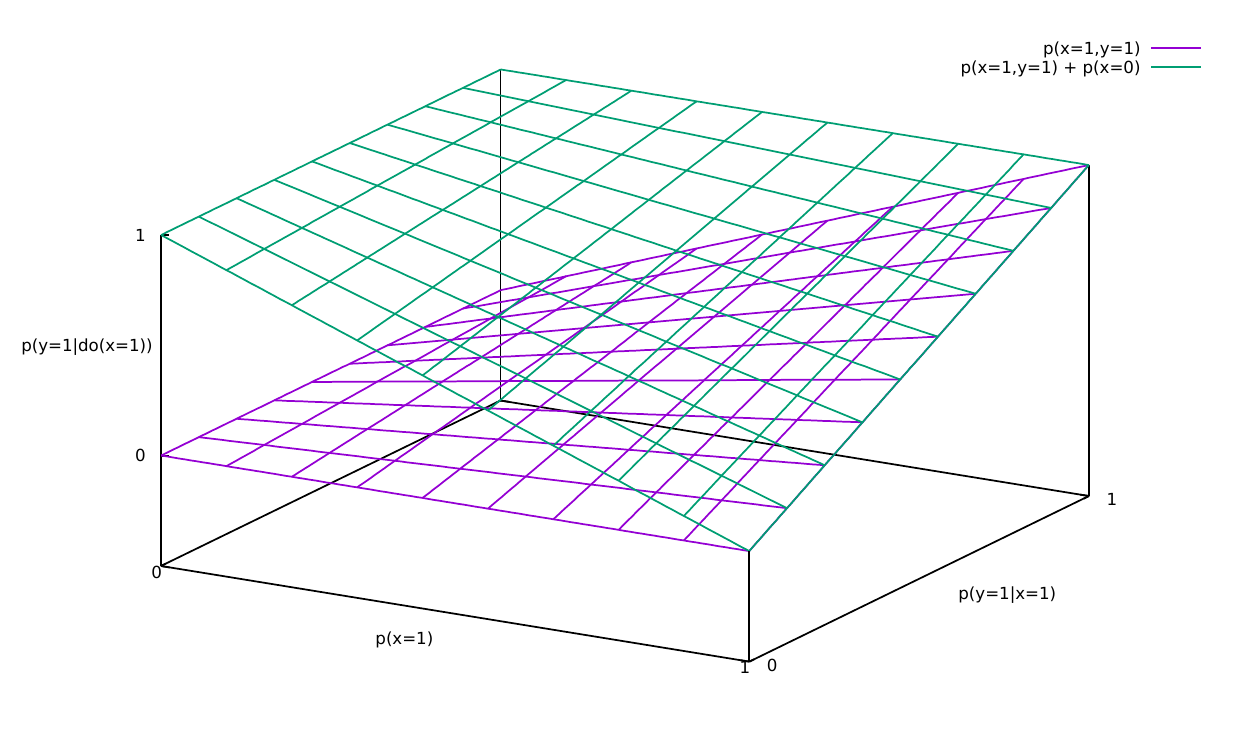}
  \caption{Illustration of the bounds in Equation~\ref{eq:bound_binary_x0} (left) and Equation~\ref{eq:bound_binary_x1} (right) of the interventional distributions in terms of the observational distribution.\label{fig:bounds}}
\end{figure}

\section{Limitations of the results in Section 4.3}
\label{sec:limitations}
An important remark, though, is that in this way we do not control for a placebo effect.
If one is specifically interested in comparing the effect of treatment with the drug with treatment with a placebo, it seems that the observational data (in which individuals were not even treated with a placebo) have nothing to offer in this scenario.
On the other hand, if the aim is to compare the effect of treatment with the drug to no treatment at all, one can save out on the control group in the RCT. 
Indeed, in this scenario there is simply no room for potential unobserved confounding in the observational data to bias the estimate.
In practice, when applying this idea, one should make sure that the observational data stem from a similar population as the treatment group. 

More generally, if the drug is not new, and not everyone in the observational data is untreated, but only a small fraction is treated, the observational data may still allow for a more accurate estimate of $p(y=1\mid\doit(x=0))$ than the interventional data if the number of interventional samples with $x=0$ is small. 
A similar situation can arise if only a small fraction of the observational population were \emph{un}treated. 
In general, the approximate ML estimator in Equation~\ref{eq:approx_ML_estimator} can be used as long as the observational data size is large enough. 
Even more generally, one can use the exact ML estimator, which can be computed numerically. 
In that case, it is not even necessary to assume that the observational data set is much larger than the interventional data set for obtaining an estimate of the causal effect based on a non-trivial combination of observational and interventional data.

\section{Proofs}\label{sec:proofs}

\begin{Thm}
\label{cdf-ind}
Let $Y$ be a `conditional' random variable with Markov kernel $P(Y|\mathbf{X})$ that takes values in 
$\mathbb{R}$ (or $[-\infty,\infty]$) and whose input $\mathbf{X}$ has values in any measurable space (e.g.\ $\mathbb{R}^M$).
Then there exists a uniformly distributed variable $E\sim U[0,1]$ that is independent of $\mathbf{X}$ and a deterministic function $F$, namely the conditional quantile function of $Y$ given $\mathbf{X}$, such that:
\begin{align}
    Y = F(E|\mathbf{X}) \quad \text{a.s.}  
\end{align}
\end{Thm}
\begin{proof}
Consider the interpolated conditional cumulative distribution function of $Y$ given $\mathbf{X}$ with $u \in [0,1]$:
\begin{align}
    G(y;u|\mathbf{x}) &:= P(Y < y|\mathbf{x}) + u \cdot P(Y=y|\mathbf{x}).
\end{align}
Furthermore, consider the conditional quantile function (cqf) of $Y$ given $\mathbf{X}$ with $e \in [0,1]$:
\begin{align}  
    F(e|\mathbf{x}) &:= \inf \{ \tilde{y} \in \mathbb{R}\,|\, G(\tilde{y};1|\mathbf{x}) \ge e \}.
\end{align}
Then take any uniformly distributed random variable $U \sim U[0,1]$ independent of $(Y,\mathbf{X})$ and define:
\begin{align}
    E &:= G(Y;U|\mathbf{X}),
\end{align}
where we plugged $Y$, $U$ and $\mathbf{X}$ into $G$.
Then one can check using standard arguments for cdf and cqf that $E$ is uniformly distributed, $E \sim U[0,1]$, which is independent of the value $\mathbf{x}$ of  $\mathbf{X}$. Furthermore, one can show that:
\begin{align}
    Y = F(E|\mathbf{X}) \quad \text{a.s.}  
\end{align}
A detailed proof can be found in \cite{Forre2021} in Appendix G.
\end{proof}

\printProofs

\section{Reduction starting from potential outcomes}\label{app:potential_outcomes}

Theorem~\ref{thm:causal_reduction} states that one can construct a reduced causal Bayesian network from a given causal Bayesian
network of a certain form. Here, we show that one can also obtain such a reduced causal Bayesian network when starting instead
from random variables
$\xi : \Omega \to \mathcal{X}$ (the observational treatment) and the potential outcomes 
$\eta : \Omega \to \mathcal{Y}^{\mathcal{X}}$.
Here, we consider $\xi$ to represent the treatment in the observational regime, while $\eta$ is a random function, with $\eta(\mathbf{x}) \in \mathcal{Y}$ the potential outcome in the regime under treatment $\mathbf{x} \in \mathcal{X}$.
Note that we do not need a random variable corresponding to the outcome in the observational regime, since this is just $\eta(\xi)$ by the consistency assumption.
We will here limit ourselves to considering a finite space $\mathcal{X}$ (that is, finitely many possible treatments) 
for simplicity of exposition.\footnote{Extending this to infinite spaces $\mathcal{X}$ is possible, but more mathematical machinery is needed in order to deal with the measure-theoretic subtleties, see for example \citep{Forre2021b}.}
In that case, one can also think of $\eta$ as a tuple $\eta = (\eta(\mathbf{x}_1), \dots, \eta(\mathbf{x}_n))$ for $\mathcal{X} = \{\mathbf{x}_1,\dots,\mathbf{x}_n\}$.

\begin{Cor}[Causal Reduction From Potential Outcomes]\label{Cor:causal_reduction_PO}
  Let $\xi$ be a random variable taking values in $\mathcal{X}$ and $\eta$ a random function $\mathcal{X} \to \mathcal{Y}$, where $\mathcal{X}$ is a finite space and $\mathcal{Y}$ a standard measurable space.
  Then there exists a causal Bayesian network $\mathcal{M}^*$ with observed variables $\mathbf{X} \in \mathcal{X}$ and $\mathbf{Y}\in\mathcal{Y}$ and a single latent confounder $\mathbf{W}\in\mathcal{X}$ (that takes values in the same space as $\mathbf{X}$), with graph as in Figure~\ref{fig:graphical_explanation} (b), and such that $\mathcal{M}^*$ entails the same observational and interventional distributions as encoded by $\xi$ and $\eta$, that is,
  $$p_{\mathcal{M}^*}(\mathbf{x},\mathbf{y}) = p_{\xi,\eta(\xi)}(\mathbf{x},\mathbf{y})$$
  $$p_{\mathcal{M}^*}(\mathbf{y} \mid \doit(\mathbf{x})) = p_{\eta(\mathbf{x})}(\mathbf{y})$$
where $p_{\xi,\eta(\xi)}$ denotes the joint distribution of observational treatment $\xi$ and outcome $\eta(\xi)$, and $p_{\eta(\mathbf{x})}$ the distribution of the potential outcome $\eta(\mathbf{x})$.
\end{Cor}
\begin{proof}
  We will construct a causal Bayesian network $\mathcal{M}$ with the graph depicted in Figure~\ref{fig:graphical_explanation} (b), that is, with a treatment variable $\mathbf{X}$, an outcome variable $\mathbf{Y}$, and a latent confounder $\mathbf{Z}$ such that it encodes the same distributions as $\xi$ and $\eta$, that is, 
  \begin{equation}\label{eq:construction_CBN_from_PO}
    \begin{cases}
      p_{\mathcal{M}}(\mathbf{x},\mathbf{y}) = p_{\xi,\eta(\xi)}(\mathbf{x},\mathbf{y}) & \\
      p_{\mathcal{M}}(\mathbf{y} \mid \doit(\mathbf{x})) = p_{\eta(\mathbf{x})}(\mathbf{y}). & \\
    \end{cases}
  \end{equation}
  We take $Z := (\xi, \eta) \in \mathcal{Z} := \mathcal{X} \times \mathcal{Y}^{\mathcal{X}}$ as the latent confounder.
  We can then define $p(\mathbf{x} \mid \mathbf{z}) := \delta_{\mathbf{z}_1}$, the delta measure centered on the first component of $\mathbf{z}$, and $p(\mathbf{y} \mid \mathbf{z}, \mathbf{x}) := \delta_{\mathbf{z}_2(\mathbf{x})}$, the delta measure centered on the second component of $\mathbf{z}$, evaluated in $\mathbf{x}$.
  It is straightforward to check that this causal Bayesian network entails Equation~\ref{eq:construction_CBN_from_PO}.

  We now apply Theorem~\ref{thm:causal_reduction} to $\mathcal{M}$ to obtain a causal Bayesian network $\mathcal{M^*}$ with a single latent confounder taking values in $\mathcal{X}$ that does the job.
  The surrogate confounder $\mathbf{W}$ constructed in the reduced causal Bayesian network has distribution $p_{\mathcal{M}^*}(\mathbf{w}) = p_{\xi}(\mathbf{w})$, while the Markov kernels for $\mathbf{X}$ and $\mathbf{Y}$ are respectively 
  $p_{\mathcal{M}^*}(\mathbf{x} \mid \mathbf{w}) = \delta_{\mathbf{w}}(\mathbf{x})$ and
  $p_{\mathcal{M}^*}(\mathbf{y} \mid \mathbf{x},\mathbf{w}) = p_{\eta(\mathbf{x}) \mid \xi}(\mathbf{y} \given \mathbf{w})$.
\end{proof}

This reduction offers a more parsimonious parameterization of the observational and interventional distributions, 
since $p(\mathbf{w})$ and $p(\mathbf{y}|\mathbf{w},\mathbf{x})$ together are $(|\mathcal{X}|-1) + (|\mathcal{Y}|-1)|\mathcal{X}|^2$ dimensional, while
$p(\xi,\eta)$ is $|\mathcal{X}|\,|\mathcal{Y}|^{\mathcal{X}}-1$ dimensional.
The dimensionality can be further reduced to $(|\mathcal{X}|-1) + 2 (|\mathcal{Y}|-1) |\mathcal{X}|$ at the cost of introducing constraints between the parameters (see Section~\ref{sec:binary:bounds}).

\section{Normalizing Flows}
\label{sec:flows}
\subsection{Background}
Normalizing flows are based on the idea of transforming samples from a simple distribution into samples from a complex distribution using the change of variable formula \citep{rezende_variational_2016, tabak_family_2013}:
\begin{align}
    p(\mathbf{x}) = p_\mathbf{Z}(f(\mathbf{x}))\left\lvert \det \left(\frac{\partial f(\mathbf{x})}{\partial \mathbf{x}}\right) \right\lvert,
    \label{eq:flow}
\end{align}
where $\mathbf{z} = f(\mathbf{x})$ is a bijective map $f: \mathcal{X} \rightarrow \mathcal{Z}$, $p_\mathbf{Z}(\mathbf{z})$ a simple prior distribution, and $\frac{\partial f(\mathbf{x})}{\partial \mathbf{x}}$ the Jacobian with respect to $\mathbf{x}$. The transformation $f(\mathbf{x})$ is commonly composed of $K$ transformations $f(\mathbf{x}) = f_K \circ \dots \circ f_1(\mathbf{x})$ to increase the overall expressivity of $f(\mathbf{x})$.
The choice of $f(\mathbf{x})$ is restricted by the computational complexity of calculating the Jacobian $\frac{\partial f(\mathbf{x})}{\partial \mathbf{x}}$. In recent years, a multitude of transformations with easy to compute Jacobians have been developed, for an overview see \cite{kobyzev_normalizing_2020, papamakarios_normalizing_2019}.

In this paper we will use neural spline flows \citep{durkan_neural_2019, dolatabadi_invertible_2020}. Neural spline flows have two major advantages: 1. A better functional flexibility than affine transformations ($\mathbf{y= sx + t}$), 2. A numerically stable, analytic inverse that has the same computational and space complexities as the forward operation.  While \cite{durkan_neural_2019} use quadratic, cubic, and rational quadratic functions whose inversion is done after solving polynomial equations, \cite{dolatabadi_invertible_2020} show that piecewise linear rational splines can perform competitively with these
methods without requiring a polynomial equation to
be solved in the inversion. Because of its reduced computational cost, we will use linear rational splines throughout this paper. 

Consider a set of monotonically increasing points $\{ (x^{(k)}, y^{(k)})\}^K_{k=0}$ called knots and a set of derivatives at each of the points $\{ d^{(k)}\}^K_{k=0}$. For each bin $[x^{(k)}, x^{(k+1)}]$ we want to find a linear rational function of the form $\frac{ax+b}{cx+d}$ that fit the given points and derivatives. 
\begin{algorithm}
\caption{\cite{fuhr_monotone_1992} Linear Rational Spline Interpolation of Monotonic data in the interval $\left[x^{(k)}, x^{(k+1)}\right]$.}
  \scriptsize
\textbf{Input:} $x^{(k)} < x^{(k+1)}$, $y^{(k)} < y^{(k+1)}$, $d^{(k)} > 0$, $d^{(k+1)} > 0$ 
\begin{algorithmic}[1]
  \STATE set $w^{(k)}$ > 0
  \STATE set $0 < \lambda^{(k)} < 1$
  \STATE $w^{(k)} = \sqrt{\frac{d^{(k)}}{d^{(k+1)}}} w^{(k)}$
  \STATE $y^{m} = \frac{w^{(k)}y^{(k)}\left(1 - \lambda^{(k)}\right) + w^{(k+1)} y^{(k+1)}\lambda^{(k)}}{w^{(k)}\left(1 - \lambda^{(k)}\right) + w^{(k+1)} \lambda^{(k)}}$
  \STATE $w^{(m)} = \left(\lambda^{(k)}w^{(k)}d^{(k)} +  \left(1 - \lambda^{(k)}\right)w^{(k+1)}d^{(k+1)}\right)\frac{x^{(k+1)}- x^{(k)}}{y^{(k+1)} - y^{(k)}}$
\end{algorithmic}
 \textbf{Return:} $\lambda^{(k)}, w^{(k)}, w^{(m)}, w^{(k+1)}, y^{(m)}$
\end{algorithm}

The values returned by Algorithm 1 are subsequentely used to express the following linear rational spline function
\begin{align}
    f(\boldsymbol{\phi}) = \begin{cases} \hspace{0.5cm} \frac{w^{(k)}y^{(k)}(\lambda^{(k)} - \boldsymbol{\phi}) + w^{(m)} y^{(m)} \boldsymbol{\phi}}{w^{(k)}(\lambda^{(k)} - \boldsymbol{\phi}) + w^{(m)} \boldsymbol{\phi}} \hspace{0.75cm} 0 \leq \boldsymbol{\phi} \leq \lambda^{(k)}\\
\frac{w^{(m)}y^{(m)}(1 - \boldsymbol{\phi}) + w^{(k+1)} y^{(k+1)} (\boldsymbol{\phi} - \lambda^{(k)})}{w^{(m)}(1 - \boldsymbol{\phi}) + w^{(k+1)} (\boldsymbol{\phi} - \lambda^{(k)})} \hspace{0.2cm} \lambda^{(k)} \leq \boldsymbol{\phi} \leq 1 \end{cases}
\end{align}
where $\boldsymbol{\phi} = (x - x^{(k)})/(x^{(k+1)} - x^{(k)})$.

 Spline flows have two hyperparameters, the boundary $B$ of the interval $[-B, B]$ and the number of bins $K$. Outside of the interval $[-B, B]$, the identity function is used. Using Equation \ref{eq:flow} we can update the parameters of the neural spline flow using maximum-likelihood estimation in combination with gradient descent. In the case where $\mathbf{x}$ has two or more dimensions, either coupling layers \citep{dinh_density_2017} or autoregressive layers \citep{papamakarios_masked_2018} can be used. 

At multiple points in this paper we are required to estimate conditional distributions, e.g. $p(\mathbf{y|x})$, where we will use conditional normalizing flows to estimate conditional probabilities. We consider the mapping $f: \mathcal{X} \times \mathcal{Y} \rightarrow \mathcal{Z}$, which is bijective in $\mathcal{Y}$ and $\mathcal{Z}$, and a simple prior distribution $p_\mathbf{Z}(\mathbf{z})$. Again, using the change of variable formula we can express the conditional distributions $p(\mathbf{y|x})$ as follows
\begin{align}
    p(\mathbf{y|x}) = p_\mathbf{Z}(f_x(\mathbf{y}))\left\lvert \det \left(\frac{\partial f_x(\mathbf{y})}{\partial \mathbf{y}}\right) \right\lvert.
\end{align}

The conditional version of the linear rational spline transformation uses a neural network to predict the derivatives $\mathbf{d}$, width $\mathbf{w}$, height $\mathbf{h}$, and $\boldsymbol{\lambda}$ from $\mathbf{x}$: $\mathbf{w, h, d}, \boldsymbol{\lambda} =  NN_{\boldsymbol{\theta}}(\mathbf{x})$.

\subsection{Parameterization using normalizing flows}

While parameterizing the functions $F$ and $G$ can be done in many different ways, we here use diffeomorphisms, i.e., differentiable mappings with a differentiable inverse. Using the change-of-variables formula, we can derive a maximum-likelihood estimator for the mappings' parameters that can be efficiently optimized through backpropagation.
In the deep learning community, those invertible and differentiable mappings are called \emph{normalizing flows}, and much recent research went into finding flexible and easily invertible mappings. See for example \cite{pawlowski_deep_2020} and \cite{khemakhem_causal_2020} for other current applications of normalizing flows to approximate nonlinear causal mechanisms. 
%In Appendix~\ref{sec:flows} we provide more details regarding the normalizing flows we have applied in our experiments.

Our flow model consists of two flows, where the first flow corresponds to $F$ and is trained using observational data, while the second flow corresponds to $G$ and is trained using observational and interventional data. In the following, we derive the loss function for observational and interventional data separately. For the remaining part of this sectoin we focus on one-dimensional treatment outcome pairs, i.e.\ $X \in \mathcal{X}= \mathbb{R}$ and $Y \in \mathcal{Y}= \mathbb{R}$, and (optionally) an $L$-dimensional observed confounder $\mathbf{C} \in \mathcal{C} = \mathbb{R}^L$.

\subsubsection{Observational data}
\label{sec:implementations_obser}

To keep the notation simple, we will henceforth suppress the dependence on $\mathbf{c}$.
According to the SCM in Equations \ref{eq:scm_X} and \ref{eq:scm_Y}, the joint likelihood $p(x,y)$ can be factorized as
$    \log p(x,y) = \log p(y|x) + \log p(x).$
We now use the following bijective transformations between observed variables $x, y$ and latent variables $u, v$
\begin{align}
u &= f_{\boldsymbol{\phi}}(x),\\
v &= g_{x, u; \boldsymbol{\theta}}(y),
\end{align}
where the functions $f_{\boldsymbol{\phi}}$ and $g_{x,u;\boldsymbol{\theta}}$ are invertible 
for all $\boldsymbol{\phi}, x, u, \boldsymbol{\theta}$. Here, $f_{\boldsymbol{\phi}} = F^{-1}$ from Equation~\ref{eq:reduced_se_X}
and $g_{x, u; \boldsymbol{\theta}}$ is the inverse of $v \mapsto G(u,v,x)$ from Equation~\ref{eq:shared_function} (for fixed $u,x,\boldsymbol{\theta}$).\footnote{In the presence of observed confounding, we simply have to replace the functions $f_{\boldsymbol{\phi}}$ and $g$ by functions $f_{\mathbf{c};\boldsymbol{\phi}}$ and $g_{x, u, \mathbf{c}; \boldsymbol{\theta}}$.}

The SCM also specifies that $u \sim \mathcal{N}(0,1)$, $v \sim \mathcal{N}(0,1)$ and $u \Indep v$.
%Without loss of generality we assume independent, standard Gaussian distributions for $u, v$: $p_U(u) = \mathcal{N}(0, 1) \Indep p_V(v) = \mathcal{N}(0, 1)$.
The transformations defined above allow us to rewrite the joint likelihood using the change of variable formula
\begin{align}
    \log p(x,y) &= \log p_V(g_{x,u; \boldsymbol{\theta}}(y)) + \log \left\lvert \frac{\partial g_{x, u; \boldsymbol{\theta}}(y)}{\partial y} \right\lvert + \log p_U(f_{\boldsymbol{\phi}}(x)) + \log \left\lvert \frac{\partial f_{\boldsymbol{\phi}}(x)}{\partial x} \right\lvert. \nonumber\\
    &= \log p_V(g_{x,f_{\boldsymbol{\phi}}(x); \boldsymbol{\theta}}(y)) + \log \left\lvert \frac{\partial g_{x, f_{\boldsymbol{\phi}}(x); \boldsymbol{\theta}}(y)}{\partial y} \right\lvert
    + \log p_U(f_{\boldsymbol{\phi}}(x)) + \log \left\lvert \frac{\partial f_{\boldsymbol{\phi}}(x)}{\partial x} \right\lvert.
\end{align}
where in the last step, we substituted $u = f_{\boldsymbol{\phi}}(x)$ into $g_{x, u; \boldsymbol{\theta}}(y)$. 
The parameters $\boldsymbol{\phi}$ and $\boldsymbol{\theta}$ are jointly updated by minimizing $\sum_{o=1}^{N_O}- \log p(x^O_o,y^O_o)$ given $N_O$ observational training samples $(x^O_1,y^O_1),\dots,(x^O_{N_O},y^O_{N_O})$.

\subsubsection{Interventional data}
\label{sec:implementations_inter}
In contrast to the observational setting, we only have to consider the conditional likelihood $p(y|\doit(x))$ in the interventional case. Since we cannot use $f_{\boldsymbol{\phi}}(x)$ to impute $u$, we instead marginalize over $u$
\begin{align}
 \log p(y\mid\doit(x)) =  \log \int p(y|\doit(x), u) p(u) du.
 \label{eq:log_p_inter}
\end{align}
Substituting the bijective mapping $v = g_{x, u; \boldsymbol{\theta}}(y)$ into Equation \ref{eq:log_p_inter}, we obtain
\begin{align}
     \log p(y\mid\doit(x)) = \log \int p_V(g_{x,u; \boldsymbol{\theta}}(y)) \left\lvert \frac{\partial g_{x, u; \boldsymbol{\theta}}(y)}{\partial y} \right\lvert p(u) du.
\end{align}
Since this is a one-dimensional integral, we can approximate it accurately numerically by means of the 
trapezoidal rule. The parameter $\boldsymbol{\theta}$ can be updated by minimizing $\sum_{i=1}^{N_I}- \log p(y^I_i|\doit(x^I_i))$ given $N_I$ interventional training samples $(x^I_1,y^I_1),\dots,(x^I_{N_I},y^I_{N_I})$.

\subsubsection{Sampling from the model}
\label{sec:how_to_sample}
\label{sec:implementation_observed_confounder}
After training we can easily generate observational and interventional samples from the trained model. 
The sampling procedure from the observational conditional distribution $p(y \mid x)$ consists of the following steps: 
\begin{align*}
  v &\sim \mathcal{N}(0,1), \\
  u &\gets f_{\boldsymbol{\phi}}(x^O),  \\
  y^O &\gets g_{x^O, u; \boldsymbol{\theta}}^{-1}(v),
\end{align*}
where we assume $x^O \in \mathbb{R}$ to be given. 
If we instead want to generate an interventional sample from $p(y \mid \doit(x))$, the sampling procedure is as follows: 
\begin{align*}
  v &\sim \mathcal{N}(0,1), \\
  u &\sim \mathcal{N}(0,1), \\
  y^I &\gets g_{x^I, u; \boldsymbol{\theta}}^{-1}(v), 
\end{align*}
where we assume $x^I \in \mathbb{R}$ to be observed. 

\section{Experiment details}
\subsection{Simulation details}
\label{sec:data_generation_details}
The generation of observational and interventional samples follows \cite{mooij_distinguishing_2016}. Instead of using Gaussian processes to model the causal mechanisms, we use two randomly initialized neural networks, $NN_1$ and $NN_2$.

\subsection{Nonlinear experiments without observed confounders}
\subsubsection{Sampling from a random distribution}
We use the following steps to generate samples from a random distribution
\begin{enumerate}
    \item $X \sim \mathcal{N}(0, 1)$
    \item sort $X$ in ascending order $= \overrightarrow{X}$
    \item sample from Gaussian Process: $F \sim \mathcal{N}(0, K_{\boldsymbol{\theta}}(\overrightarrow{X}) + \sigma^2 I)$, where for the kernel $K_{\boldsymbol{\theta}}$ we use the squared exponential covariance function with automatic relevance determination kernel
    \item use the trapezoidal rule to calculate the cumulative integral of $\exp(F)$, we obtain a vector $G$ where each element $G_i$ corresponds to $G_i = \int_{\overrightarrow{X_1}}^{\overrightarrow{X_i}} \exp(F(x)) dx$
\end{enumerate}

We will denote this whole sampling procedure by $G \sim \mathcal{RD}(\boldsymbol{\theta}, \sigma)$, where we sample $\boldsymbol{\theta}$ from a Gamma distribution $\Gamma(a,b)$ and set $\sigma = 0.0001$.

\subsubsection{Generate observational and interventional data}
1. Sample from latent variables 
\begin{align}
    \boldsymbol{\theta}_{E_X} &\sim \Gamma(a_{E_X}, b_{E_X}),\\
    \boldsymbol{\theta}_{E_Y} &\sim \Gamma(a_{E_Y}, b_{E_Y}),\\
    \mathbf{\boldsymbol{\theta}}_\mathbf{Z} &\sim \Gamma(a_\mathbf{Z}, b_\mathbf{Z}),\\
    E_X &\sim \mathcal{RD}(\boldsymbol{\theta}_{E_X}, \sigma),\\
    E_Y &\sim \mathcal{RD}(\boldsymbol{\theta}_{E_Y}, \sigma),\\
    \mathbf{Z} &\sim \mathcal{RD}(\mathbf{\boldsymbol{\theta}}_\mathbf{Z}, \sigma).
\end{align}
2. Generate $X_\text{observational}$
\begin{align}
     X_\text{observational} = NN_1(E_X, \mathbf{Z}).
\end{align}
3. Normalize $X_\text{observational}$
\begin{align}
    X_\text{observational} = \frac{X_\text{observational} - \mathbb{E}\left[X_\text{observational}\right]}{\sqrt{\mathbb{V}[X_\text{observational}]}}.
\end{align}

4. Generate $Y_\text{observational}$
\begin{align}
    Y_\text{observational} &= NN_2(X_\text{observational}, E_Y, \mathbf{Z}).
\end{align}

5. Sample from latent variables
\begin{align}
    E_Y \sim \mathcal{RD}(\boldsymbol{\theta}_{E_Y}, \sigma)\\
    \mathbf{Z} \sim \mathcal{RD}(\mathbf{\boldsymbol{\theta}}_\mathbf{Z}, \sigma)
\end{align}

6. Generate $X_\text{interventional}$
\begin{align}
    \boldsymbol{\theta}_{X} \sim \Gamma(a_{X}, b_{X}),\\
     X_\text{interventional} \sim \mathcal{RD}(\boldsymbol{\theta}_X, \sigma).
\end{align}

7. Normalize $X_\text{interventional}$
\begin{align}
    X_\text{interventional} = \frac{X_\text{interventional} - \mathbb{E}\left[X_\text{interventional}\right]}{\sqrt{\mathbb{V}[X_\text{interventional}]}}.
\end{align}

8. Generate $Y_\text{interventional}$
\begin{align}
    Y_\text{inter} &= NN_2(X_\text{inter}, E_Y, \mathbf{Z}).
\end{align}

9. Generate noise
\begin{align}
    \epsilon_{x, \text{observational}} &\sim \mathcal{N}(0,1),\\
    \epsilon_{x, \text{interventional}} &\sim \mathcal{N}(0,1),\\
    \boldsymbol{\theta}_{\epsilon_x} &\sim \Gamma(a_{\epsilon_x}, b_{\epsilon_x}),\\
    \epsilon_{y, \text{observational}} &\sim \mathcal{N}(0,1),\\
    \epsilon_{y, \text{interventional}} &\sim \mathcal{N}(0,1),\\
    \boldsymbol{\theta}_{\epsilon_y} &\sim \Gamma(a_{\epsilon_y}, b_{\epsilon_y}).
\end{align}

10. Add noise
\begin{align}
    X'_\text{observational} &= X_\text{observational} + \boldsymbol{\theta}_{\epsilon_x} \epsilon_{x, \text{observational}},\\
    X'_\text{interventional} &= X_\text{interventional} + \boldsymbol{\theta}_{\epsilon_x} \epsilon_{x, \text{interventional}},\\
    Y'_\text{observational} &= Y_\text{observational} + \boldsymbol{\theta}_{\epsilon_y} \epsilon_{y, \text{observational}},\\
    Y'_\text{interventional} &= Y_\text{interventional} + \boldsymbol{\theta}_{\epsilon_y} \epsilon_{y, \text{interventional}}.
\end{align}

11. Normalize $Y$ jointly
\begin{align}
    Y' &= [Y'_\text{observational}, Y'_\text{interventional}],\\
    Y'_\text{observational} &= \frac{Y'_\text{observational} - \mathbb{E}\left[Y'\right]}{\sqrt{\mathbb{V}[Y']}},\\
    Y'_\text{interventional} &= \frac{Y'_\text{interventional} - \mathbb{E}\left[Y'\right]}{\sqrt{\mathbb{V}[Y']}}.
\end{align}

The two neural networks $NN_1$ and $NN_2$ are Multi-layer perceptrons with a single hidden layer. The hidden layer contains 1024 units. The input layer and the hidden layer use a ReLU activation function. The weights and biases for both neural networks are uniformly sampled from the interval $[-1, 1]$. We choose the other simulation parameters as follows: $a_{E_X} =  a_{E_Y} =  a_{Z} = a_{X} = 5$, $a_{\epsilon_x} = a_{\epsilon_y} = 2$, $b_{E_X} = b_{E_Y} = b_{Z} = b_{X} = b_{\epsilon_x} = b_{\epsilon_y} = 0.1$, $\sigma=0.0001$

\subsection{Simulation details: Nonlinear experiments with observed confounders}

In order to simulate data with additional observed confounders, we first generate $\mathbf{C}$
\begin{align}
    \boldsymbol{\theta}_C = \Gamma(a_C, b_C),\\
    \mathbf{C} \sim \mathcal{RD}(\boldsymbol{\theta}_C, \sigma),
\end{align}
where $a_C=10$ and $b_C=1$. In addition, we modify steps 2,4 and 8 as follows
\begin{align}
    X_\text{observational} &= NN_1(E_X, \mathbf{Z}, \mathbf{C}),\\
    Y_\text{observational} &= NN_2(X_\text{observational}, E_Y, \mathbf{Z}, \mathbf{C}),\\
    Y_\text{inter} &= NN_2(X_\text{inter}, E_Y, \mathbf{Z}, \mathbf{C}).
\end{align}

\subsection{Data selection, hyperparameters and training}
Following the process described above, we simulate 100 datasets while varying the number of dimensions $K$ of the unobserved confounder $\mathbf{Z}$ and the random seed that is, among others, controlling the initialization of the neural networks used to model the causal mechanisms. We choose $K$ between 1 and 10 since for $K>10$, the joint distribution $p(x,y)$ becomes increasingly Gaussian due to the central limit theorem. Next, we manually select ten datasets with the smallest overlap of observational and interventional samples to select cases with ``strong'' confounding. Note that we choose these ten datasets before training a single flow model. A scatter plot of 1000 observational and 1000 interventional samples for each of the ten datasets can be found in Appendix \ref{sec:all_results}.

Motivated by the work of \cite{oliver_realistic_2019} on the realistic evaluation of semi-supervised learning algorithms, we use the same number of samples for training and validation. In every case, we use 1000 interventional samples for testing. To compare the performance of the three flow models, we calculate the negative log-likelihood averaged over the test set,
$-\frac{1}{N_I}\sum_{i=1}^{N_T} \log p(y^T_i\mid\doit(x^T_i))$, where the test set 
consists of $N_T$ samples $(x^T_1,y^T_1),\dots,(x^T_{N_T},y^T_{N_T})$ from $p(y \mid \doit(x))$.
To have a fair comparison, the same training procedure, architecture, optimizer, and hyperparameters are used for all flow models in all experiments. We use Adam \citep{kingma_adam_2017} with a learning rate of $0.001$ and the default values for $\beta_1, \beta_2$. We train for 10,000 epochs. The training is terminated early when the validation loss did not improve for 1,000 epochs. We perform full batch gradient descent, where we alternate between batches of observational and interventional samples for the third flow model. For the linear rational spline flows, we use 32 bins and set the bound $B=6$. We use a fully connected neural network with two hidden layers and ReLU activations for the conditional version of the linear rational spline flows. 

\subsection{Nonlinear experiment results without observed confounders}

In the following, for each of the 15 datasets we present the original training data (top), as well as interventional samples from a flow model trained with 50 interventional samples (center), and samples from a flow model trained with 50 interventional samples and 1000 observational samples (bottom). The samples are generated as described in Section \ref{sec:how_to_sample}.

In addition we show performances measured in terms of negative log-likelihood on the observational and the interventional test sets, respectively. We compare a flow model trained with 1000 observational samples, a flow model trained with {50, 100, 250, 500, 750, 1000} interventional samples, and a flow model trained with both 1000 observational samples and {50, 100, 250, 500, 750, 1000} interventional samples. All flow models are evaluated on 1000 interventional samples from the test set. Last, we compare a flow model trained with 1000 observational samples, a flow model trained with {50, 100, 250, 500, 750, 1000} interventional samples, and a flow model trained with both 1000 observational samples and {50, 100, 250, 500, 750, 1000} interventional samples. All flow models are evaluated on 1000 observational samples from the test set. We report the mean and standard error for ten runs of each experiment.

\label{sec:all_results}
\twocolumn
\subsubsection{Dataset 1: $\#$ of confounders = 1, random seed = 6}
\begin{figure}[H]
\centering
\includegraphics[scale=0.5]{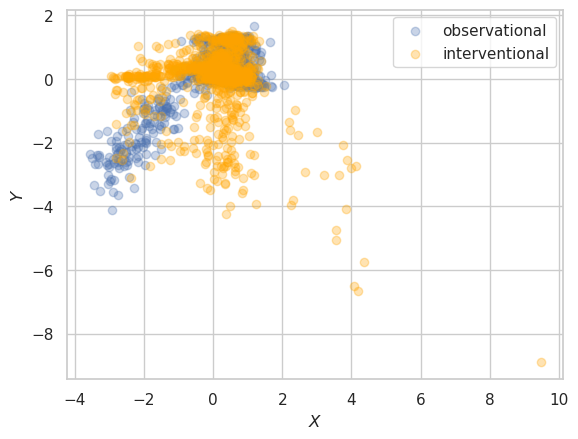}
\includegraphics[scale=0.5]{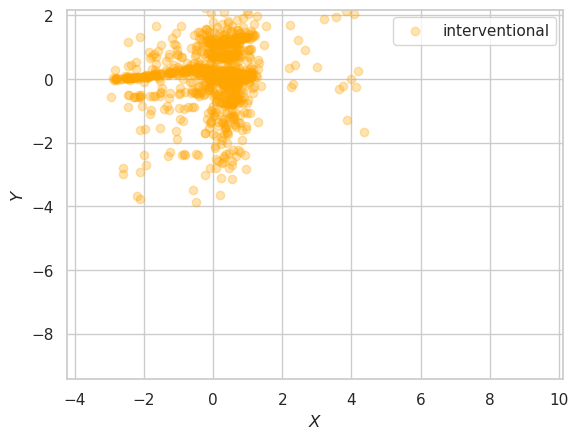}
\includegraphics[scale=0.5]{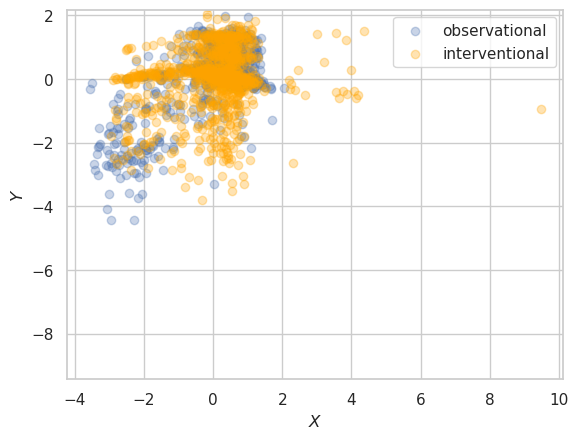}
\caption{Dataset 1: Interventional and observational samples. (Top) training data. (Center) samples from flow model trained with only interventional data. (Bottom) samples from flow model trained with observational and interventional data.}
\end{figure}
\begin{figure}[H]
\centering
\includegraphics[scale=0.5]{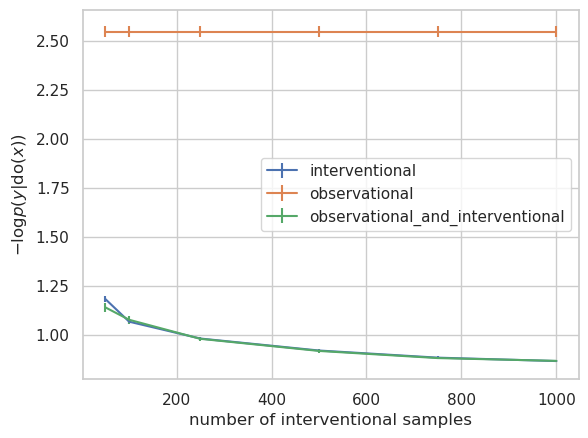}
\includegraphics[scale=0.5]{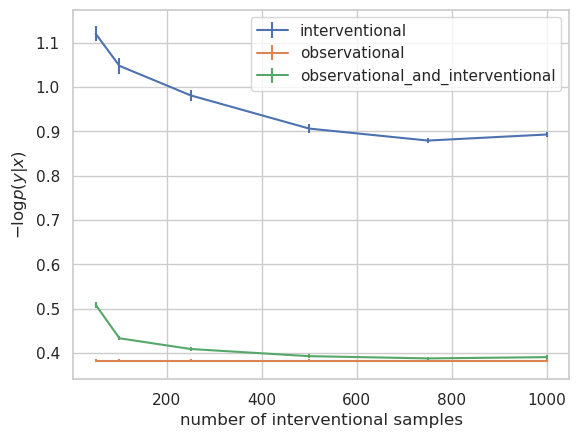}
\caption{Dataset 1: Performances measured in terms of negative log-likelihood on the interventional (top) and the observational (bottom) test sets.}
\end{figure}
\subsubsection{Dataset 2: $\#$ of confounders = 1, random seed $=$ 8}
\begin{figure}[H]
\centering
\includegraphics[scale=0.5]{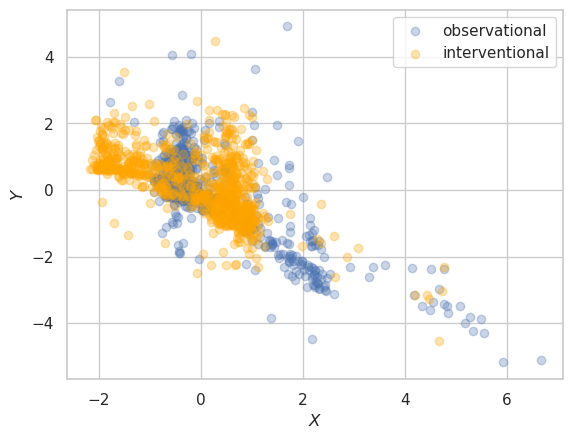}
\includegraphics[scale=0.5]{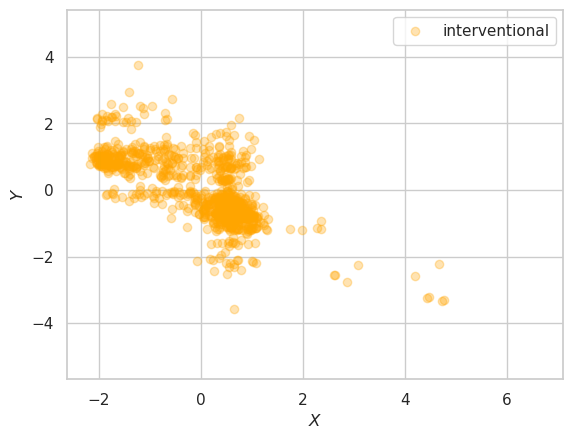}
\includegraphics[scale=0.5]{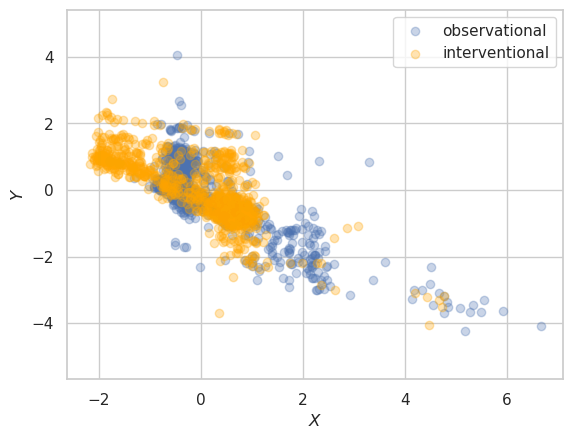}
\caption{Dataset 2: Interventional and observational samples. (Top) training data. (Center) samples from flow model trained with only interventional data. (Bottom) samples from flow model trained with observational and interventional data.}
\end{figure}
\begin{figure}[H]
\centering
\includegraphics[scale=0.5]{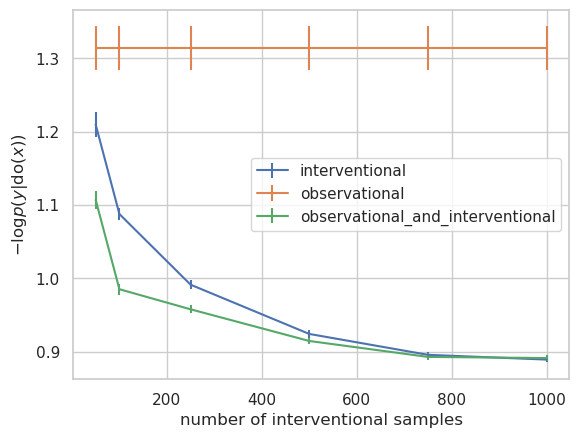}
\includegraphics[scale=0.5]{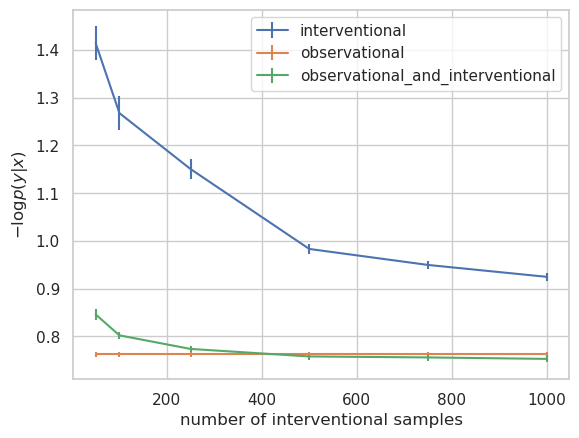}
\caption{Dataset 2: Performances measured in terms of negative log-likelihood on the interventional (top) and the observational (bottom) test sets.}
\end{figure}
\newpage
\subsubsection{Dataset 3: $\#$ of confounders = 2, random seed $=$ 7}
\begin{figure}[H]
\centering
\includegraphics[scale=0.5]{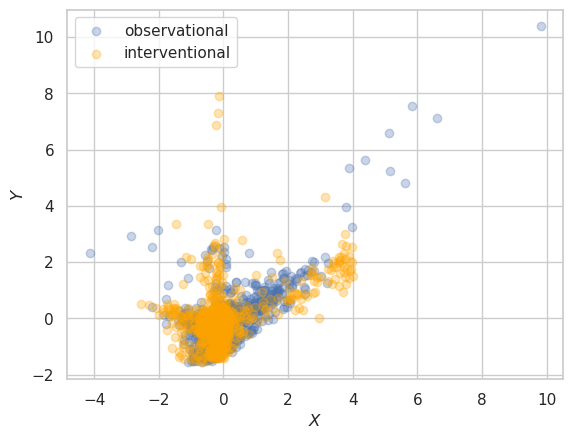}
\includegraphics[scale=0.5]{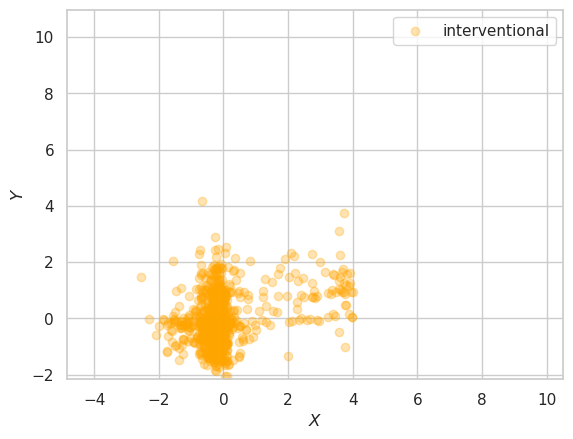}
\includegraphics[scale=0.5]{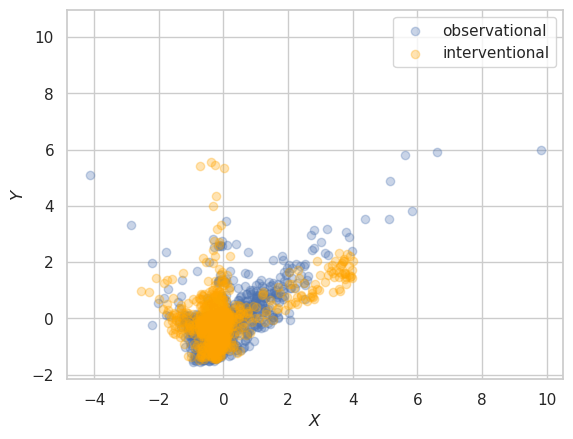}
\caption{Dataset 3: Interventional and observational samples. (Top) training data. (Center) samples from flow model trained with only interventional data. (Bottom) samples from flow model trained with observational and interventional data.}
\end{figure}
\begin{figure}[H]
\centering
\includegraphics[scale=0.5]{nonlinear_data/comparison_gp_seed_7_numZ_2.png}
\includegraphics[scale=0.5]{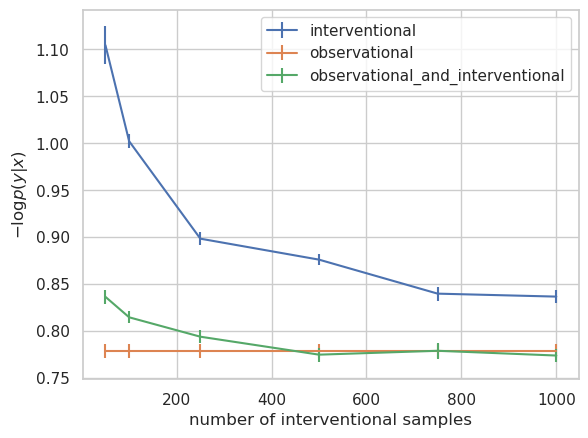}
\caption{Dataset 3: Performances measured in terms of negative log-likelihood on the interventional (top) and the observational (bottom) test sets.}
\end{figure}
\newpage
\subsubsection{Dataset 4: 3 confounders, random seed = 1}
\begin{figure}[H]
\centering
\includegraphics[scale=0.5]{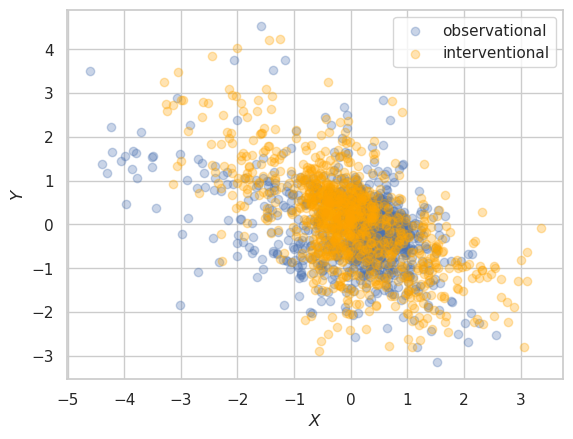}
\includegraphics[scale=0.5]{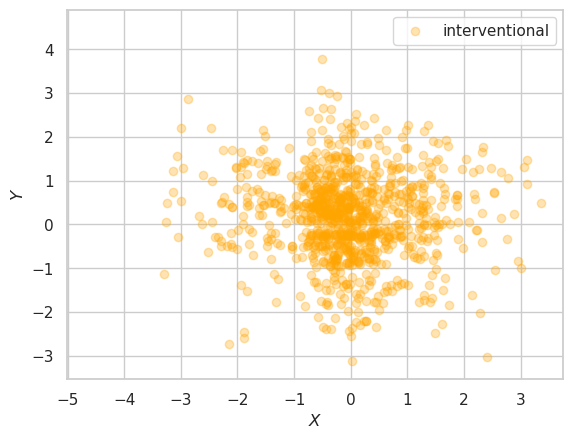}
\includegraphics[scale=0.5]{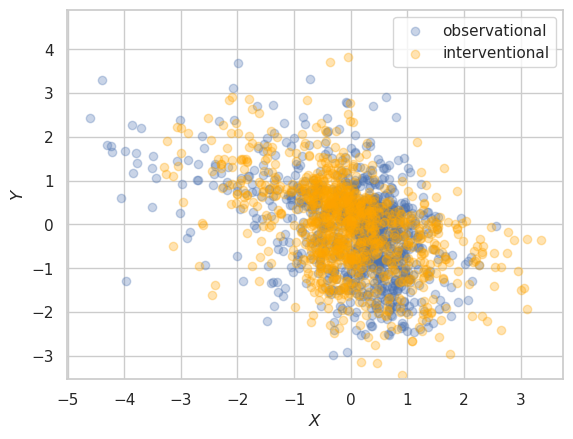}
\caption{Dataset 4: Interventional and observational samples. (Top) training data. (Center) samples from flow model trained with only interventional data. (Bottom) samples from flow model trained with observational and interventional data.}
\end{figure}
\begin{figure}[H]
\centering
\includegraphics[scale=0.5]{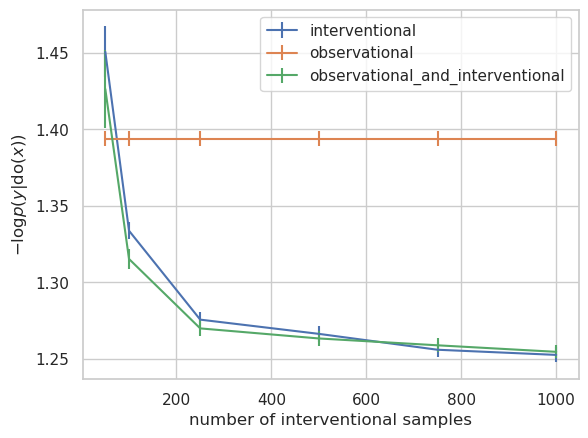}
\includegraphics[scale=0.5]{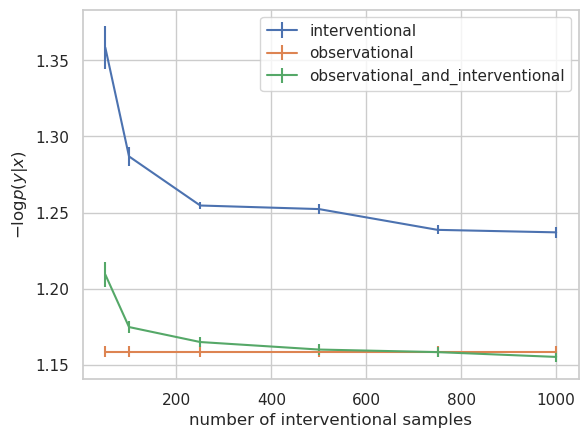}
\caption{Dataset 4: Performances measured in terms of negative log-likelihood on the interventional (top) and the observational (bottom) test sets.}
\end{figure}
\newpage
\subsubsection{Dataset 5: $\#$ of confounders = 4, random seed $=$ 0}
\begin{figure}[H]
\centering
\includegraphics[scale=0.5]{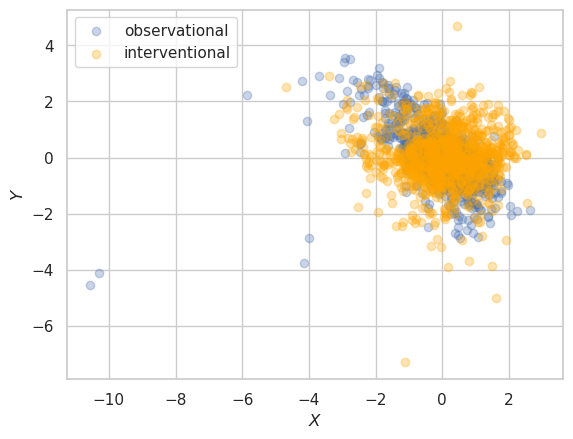}
\includegraphics[scale=0.5]{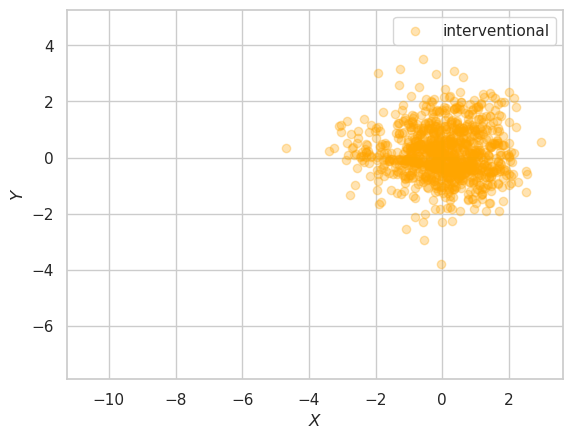}
\includegraphics[scale=0.5]{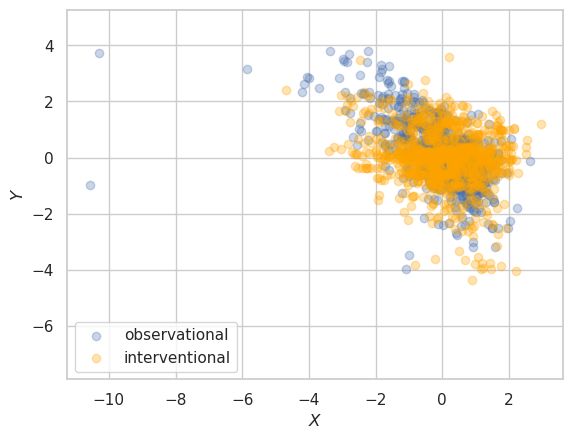}
\caption{Dataset 5: Interventional and observational samples. (Top) training data. (Center) samples from flow model trained with only interventional data. (Bottom) samples from flow model trained with observational and interventional data.}
\end{figure}
\begin{figure}[H]
\centering
\includegraphics[scale=0.5]{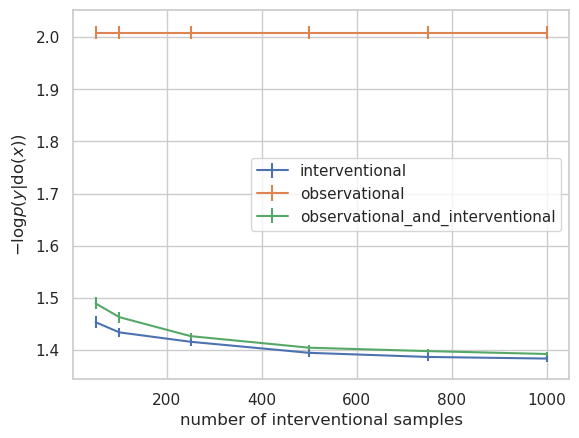}
\includegraphics[scale=0.5]{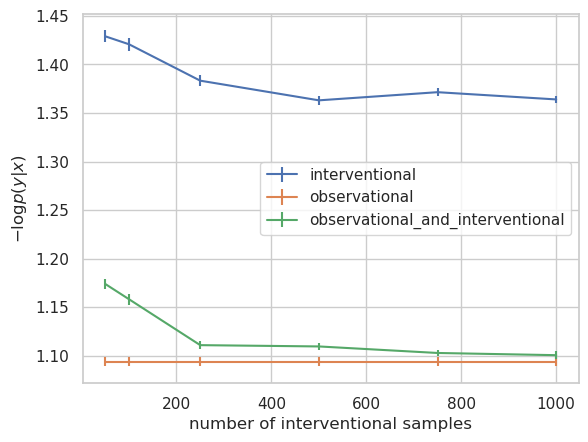}
\caption{Dataset 5: Performances measured in terms of negative log-likelihood on the interventional (top) and the observational (bottom) test sets.}
\end{figure}
\newpage
\subsubsection{Dataset 6: $\#$ of confounders = 4, random seed $=$ 7}
\begin{figure}[H]
\centering
\includegraphics[scale=0.5]{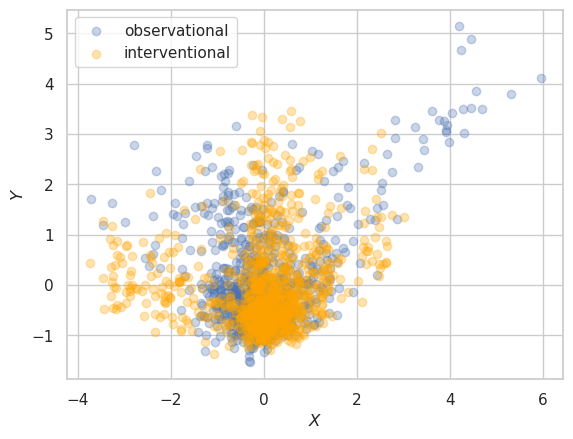}
\includegraphics[scale=0.5]{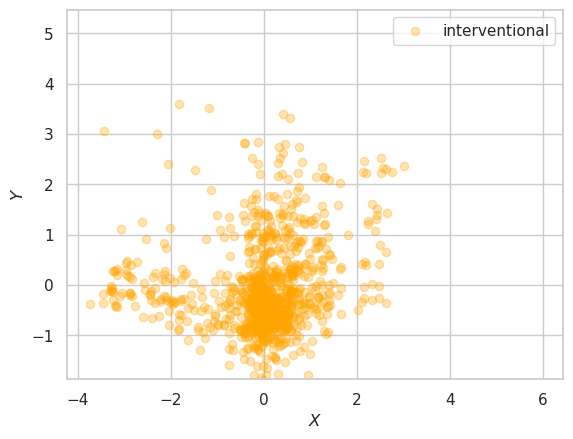}
\includegraphics[scale=0.5]{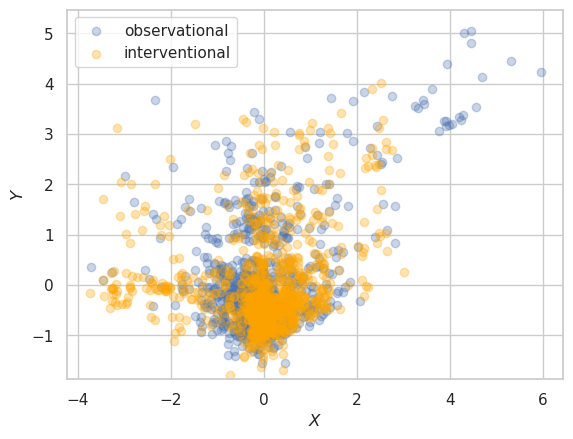}
\caption{Dataset 6: Interventional and observational samples. (Top) training data. (Center) samples from flow model trained with only interventional data. (Bottom) samples from flow model trained with observational and interventional data.}
\end{figure}
\begin{figure}[H]
\centering
\includegraphics[scale=0.5]{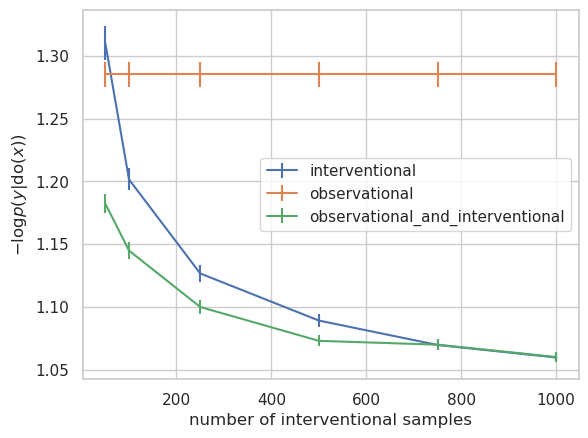}
\includegraphics[scale=0.5]{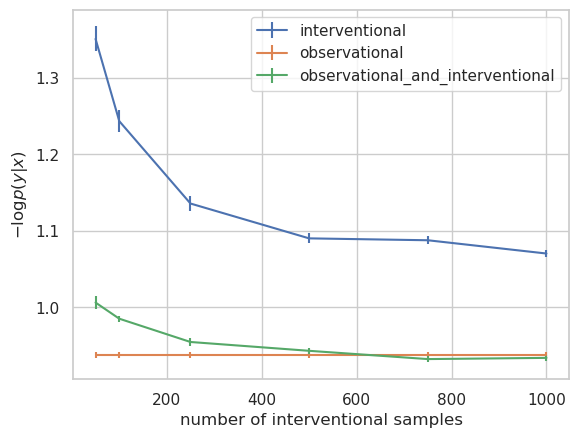}
\caption{Dataset 6: Performances measured in terms of negative log-likelihood on the interventional (top) and the observational (bottom) test sets.}
\end{figure}
\newpage
% \subsubsection{Dataset 7: $\#$ of confounders = 5, random seed $=$ 5}
% \begin{figure}[H]
% \centering
% \includegraphics[scale=0.4]{nonlinear_data/scm_nn_obser_1000_inter_1000_gp_seed_5_numZ_5.png}
% \includegraphics[scale=0.4]{nonlinear_data/samples_spline_only_inter_gp_seed_5_numZ_5_num_obser1000_num_inter_50_seed_1.png}
% \includegraphics[scale=0.4]{nonlinear_data/samples_spline_obser_and_inter_sum_gp_seed_5_numZ_5_num_obser1000_num_inter_50_seed_1.png}
% \caption[Dataset 7: Interventional and observational samples.]{Dataset 7: Interventional and observational samples. (Top) training data. (Center) samples from flow model trained with only interventional data. (Bottom) samples from flow model trained with observational and interventional data.}
% \end{figure}
% \begin{figure}[H]
% \centering
% \includegraphics[scale=0.5]{nonlinear_data/comparison_gp_seed_5_numZ_5.png}
% \includegraphics[scale=0.5]{nonlinear_data/obser_comparison_gp_seed_5_numZ_5.png}
% \caption{Dataset 7: Performances measured in terms of negative log-likelihood on the interventional (top) and the observational (bottom) test sets.}
% \end{figure}
% \newpage
\subsubsection{Dataset 7: $\#$ of confounders = 5, random seed $=$ 5}
\begin{figure}[H]
\centering
\includegraphics[scale=0.49]{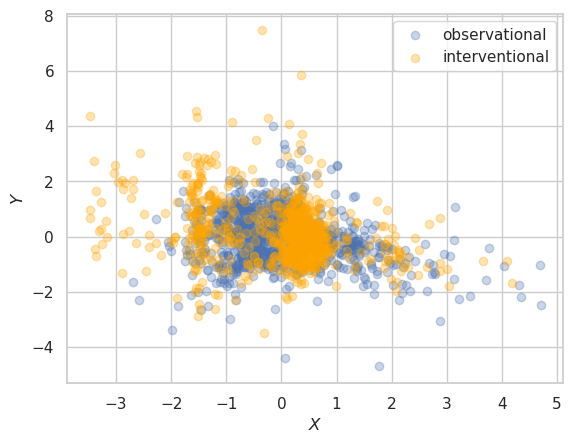}
\includegraphics[scale=0.5]{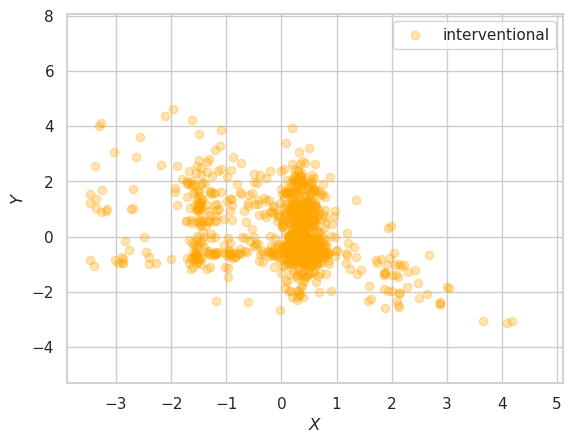}
\includegraphics[scale=0.5]{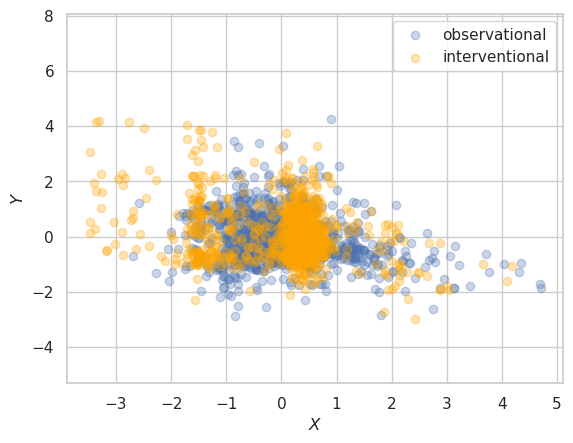}
\caption[Dataset 7: Interventional and observational samples.]{Dataset 7: Interventional and observational samples. (Top) training data. (Center) samples from flow model trained with only interventional data. (Bottom) samples from flow model trained with observational and interventional data.}

\end{figure}

\begin{figure}[H]
\centering
\includegraphics[scale=0.5]{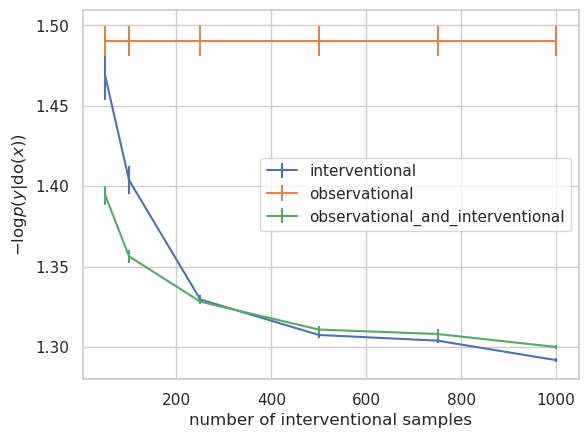}
\includegraphics[scale=0.5]{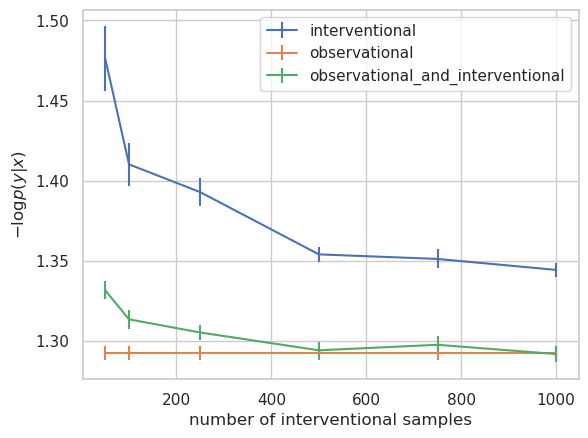}
\caption{Dataset 7: Performances measured in terms of negative log-likelihood on the interventional (top) and the observational (bottom) test sets.}

\end{figure}
\newpage

\subsubsection{Dataset 8: $\#$ of confounders = 5, random seed $=$ 9}
\begin{figure}[H]
\centering
\includegraphics[scale=0.5]{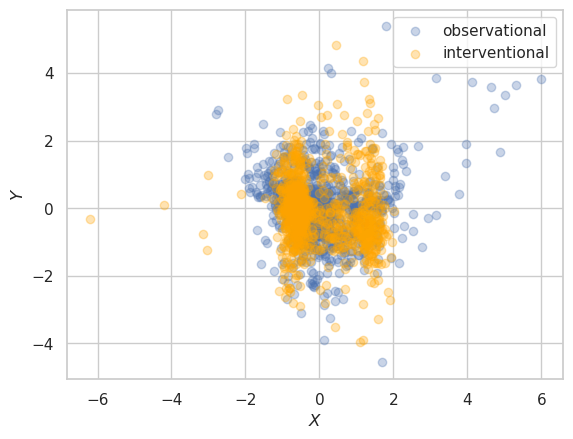}
\includegraphics[scale=0.5]{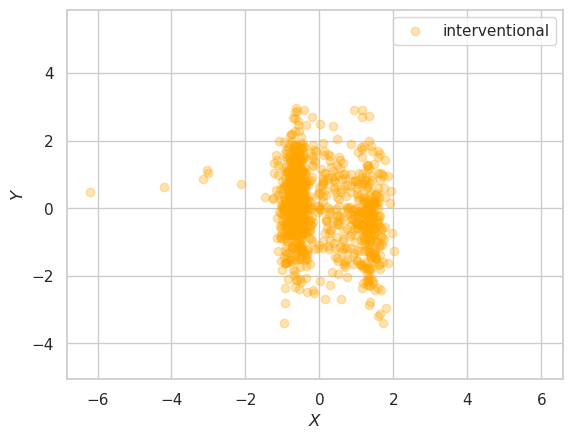}
\includegraphics[scale=0.5]{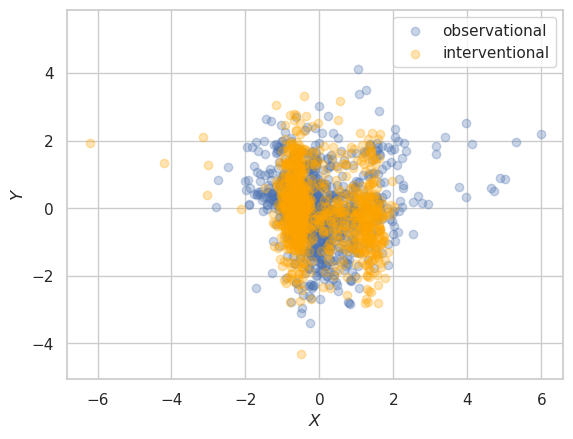}
\caption[Dataset 8: Interventional and observational samples.]{Dataset 8: Interventional and observational samples. (Top) training data. (Center) samples from flow model trained with only interventional data. (Bottom) samples from flow model trained with observational and interventional data.}

\end{figure}

\begin{figure}[H]
\centering
\includegraphics[scale=0.5]{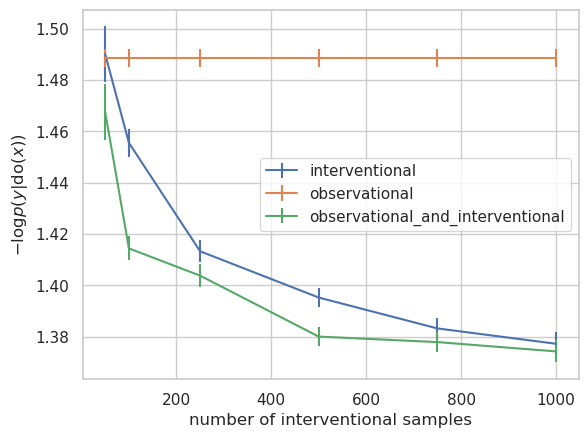}
\includegraphics[scale=0.5]{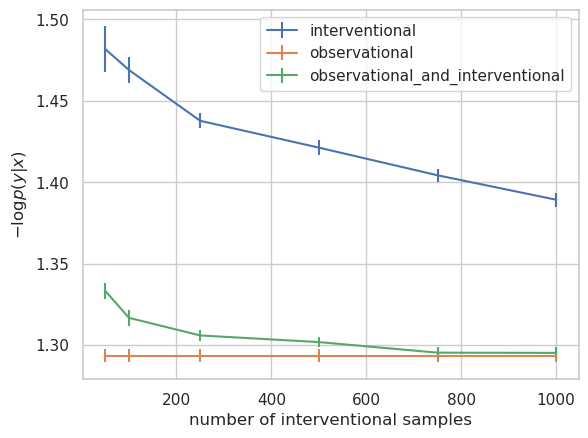}
\caption{Dataset 8: Performances measured in terms of negative log-likelihood on the interventional (top) and the observational (bottom) test sets.}

\end{figure}
\newpage

\subsubsection{Dataset 9: $\#$ of confounders = 7, random seed $=$ 0}
\begin{figure}[H]
\centering
\includegraphics[scale=0.5]{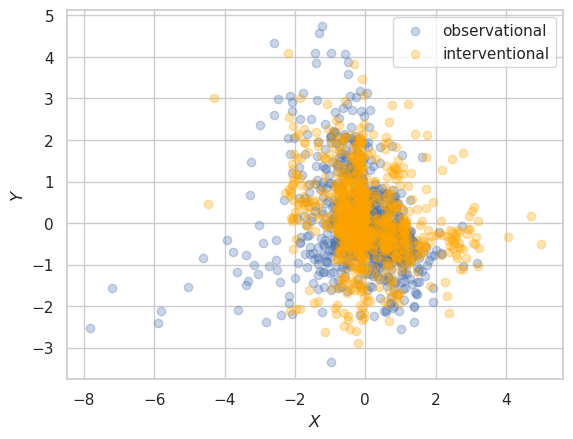}
\includegraphics[scale=0.5]{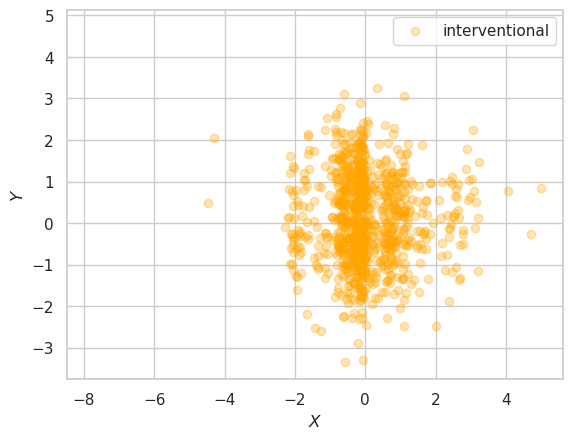}
\includegraphics[scale=0.5]{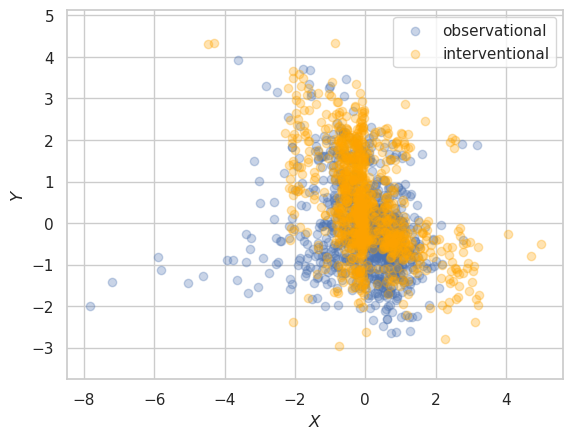}
\caption[Dataset 9: Interventional and observational samples.]{Dataset 9: Interventional and observational samples. (Top) training data. (Center) samples from flow model trained with only interventional data. (Bottom) samples from flow model trained with observational and interventional data.}

\end{figure}

\begin{figure}[H]
\centering
\includegraphics[scale=0.5]{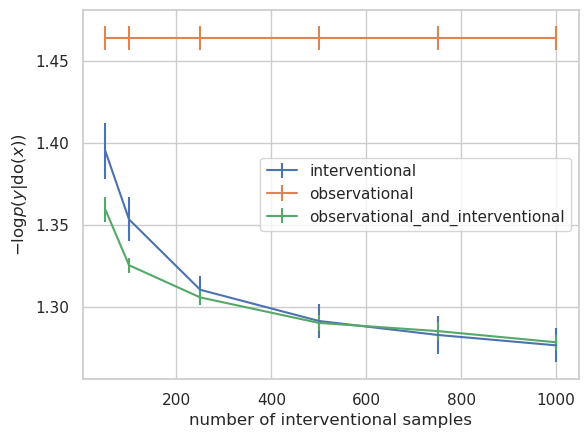}
\includegraphics[scale=0.5]{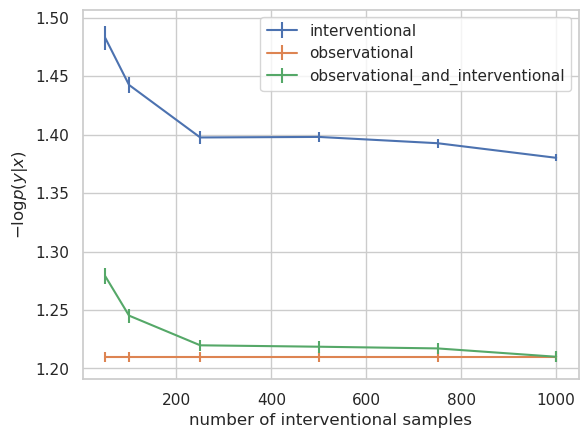}
\caption{Dataset 9: Performances measured in terms of negative log-likelihood on the interventional (top) and the observational (bottom) test sets.}

\end{figure}
\newpage
\subsubsection{Dataset 10: $\#$ of confounders = 7, random seed $=$ 5}
\begin{figure}[H]
\includegraphics[scale=0.5]{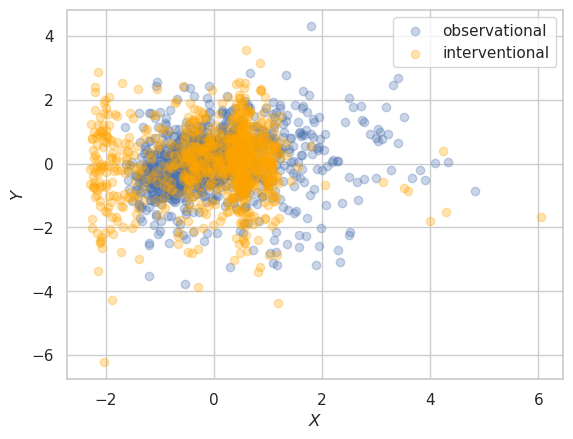}
\includegraphics[scale=0.5]{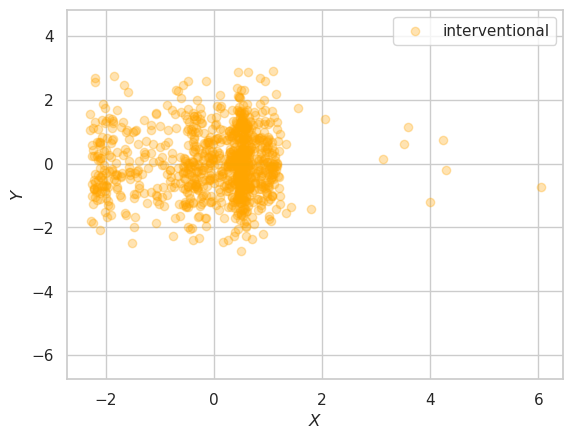}
\includegraphics[scale=0.5]{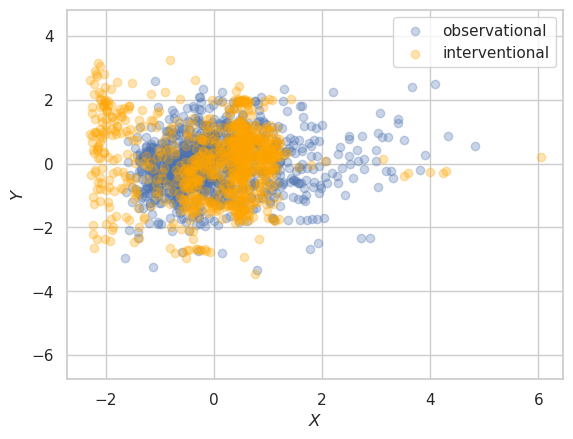}
\caption[Dataset 10: Interventional and observational samples.]{Dataset 10: Interventional and observational samples. (Top) training data. (Center) samples from flow model trained with only interventional data. (Bottom) samples from flow model trained with observational and interventional data.}
\end{figure}
\begin{figure}[H]
\centering
\includegraphics[scale=0.5]{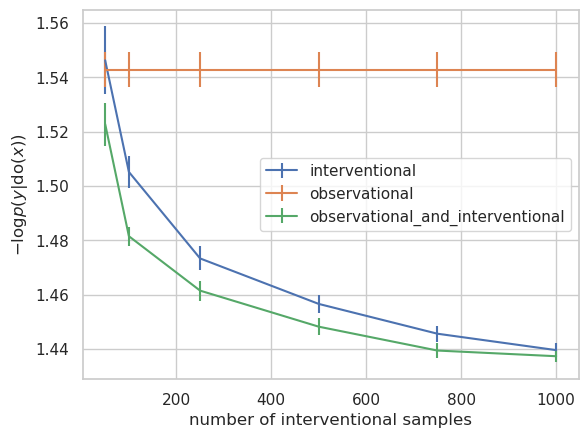}
\includegraphics[scale=0.5]{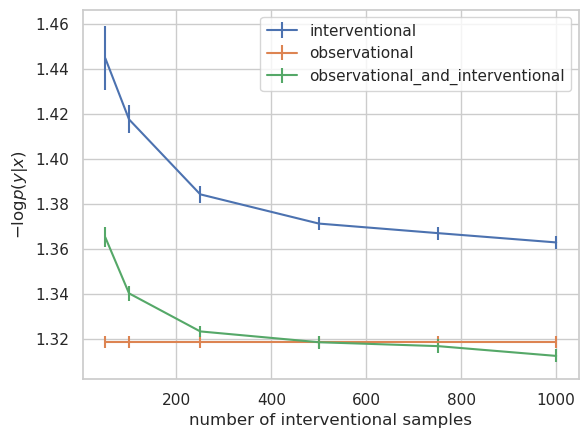}
\caption{Dataset 10: Performances measured in terms of negative log-likelihood on the interventional (top) and the observational (bottom) test sets.}
\end{figure}
\subsubsection{Dataset 11: $\#$ of latent confounders = 1, $\#$ of observed confounders = 3, random seed = 7}
\label{sec:all_results_with_confounder}
\begin{figure}[H]
\centering
\includegraphics[scale=0.45]{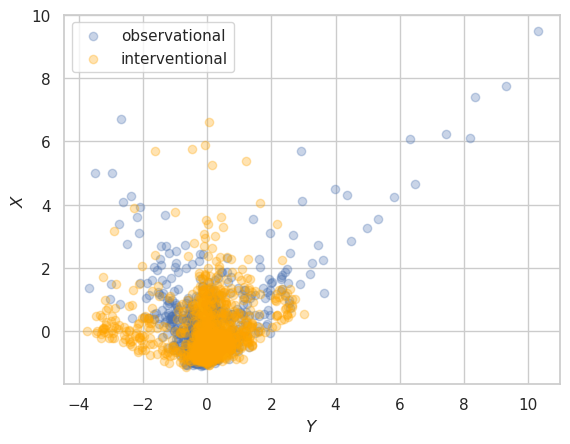}
\includegraphics[scale=0.45]{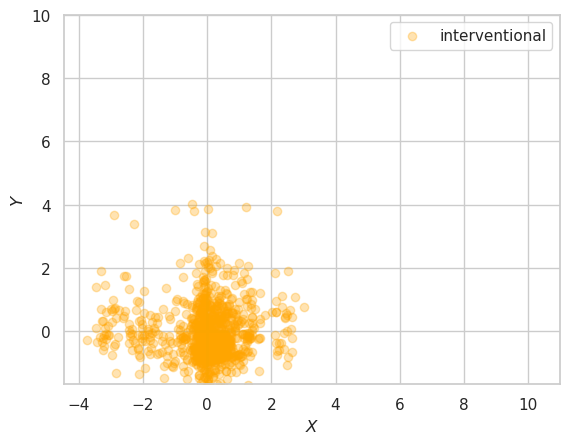}
\includegraphics[scale=0.45]{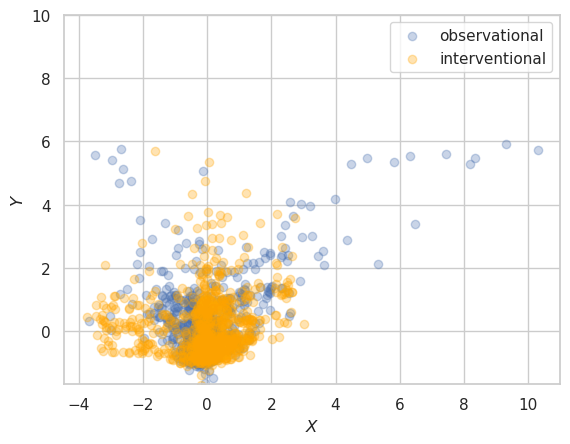}
\caption[Dataset 11: Interventional and observational samples.]{Dataset 10: Interventional and observational samples. (Top) training data. (Center) samples from flow model trained with only interventional data. (Bottom) samples from flow model trained with observational and interventional data.}
\end{figure}
\begin{figure}[H]
\centering
\includegraphics[scale=0.5]{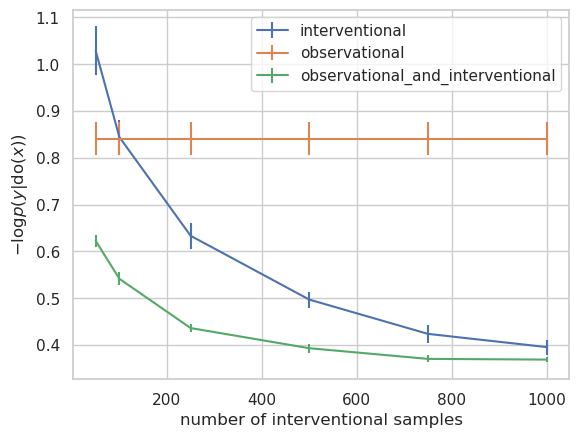}
\includegraphics[scale=0.5]{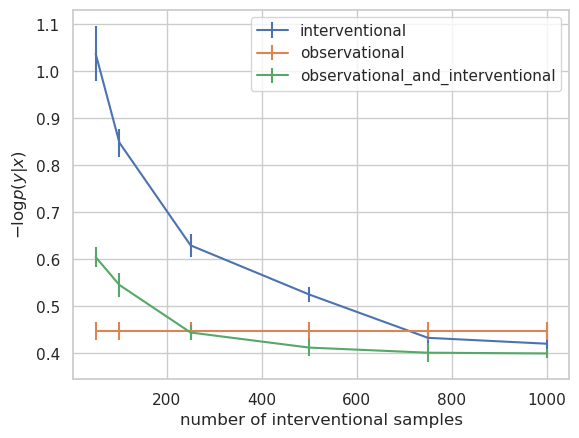}
\caption{Dataset 11: Performances measured in terms of negative log-likelihood on the interventional (top) and the observational (bottom) test sets.}
\end{figure}
\subsubsection{Dataset 12: $\#$ of latent confounders = 1, $\#$ of observed confounders = 3, random seed = 9}
\begin{figure}[H]
\centering
\includegraphics[scale=0.45]{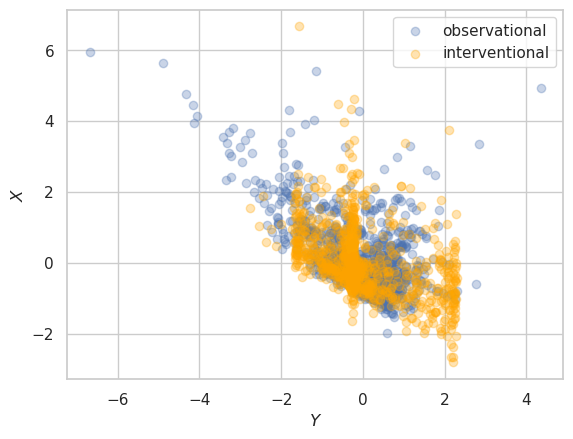}
\includegraphics[scale=0.45]{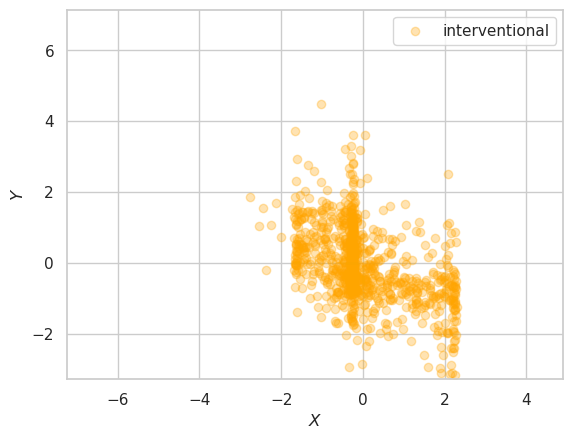}
\includegraphics[scale=0.45]{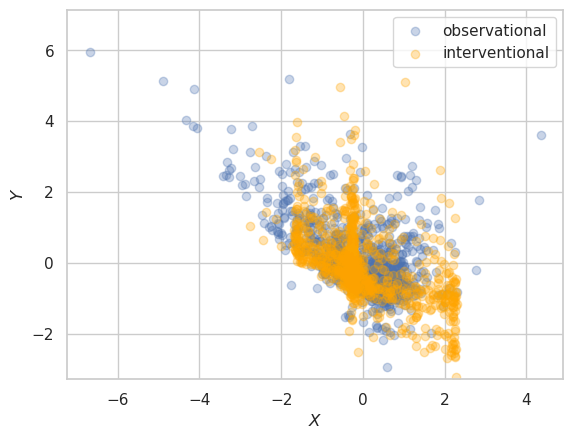}
\caption[Dataset 12: Interventional and observational samples.]{Dataset 12: Interventional and observational samples. (Top) training data. (Center) samples from flow model trained with only interventional data. (Bottom) samples from flow model trained with observational and interventional data.}
\end{figure}
\begin{figure}[H]
\centering
\includegraphics[scale=0.5]{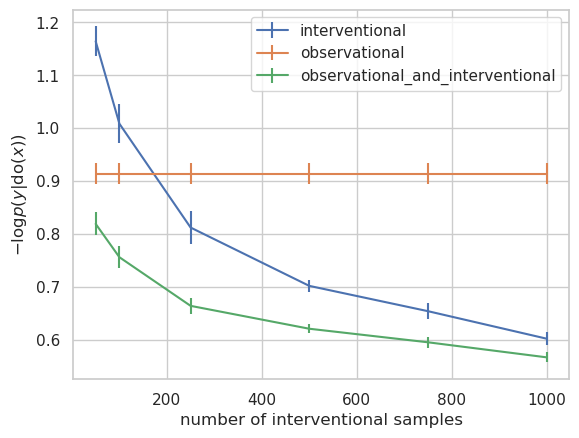}
\includegraphics[scale=0.5]{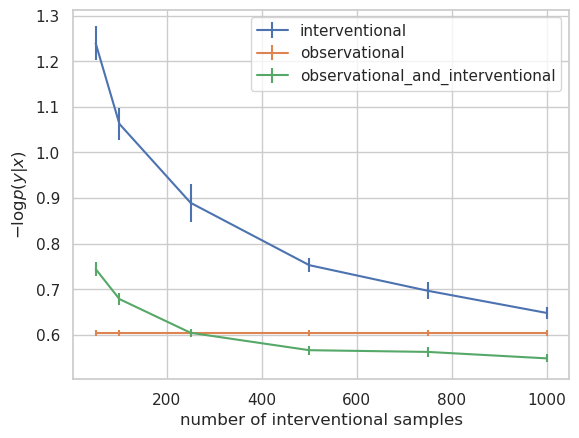}
\caption{Dataset 12: Performances measured in terms of negative log-likelihood on the interventional (top) and the observational (bottom) test sets.}
\end{figure}
\subsubsection{Dataset 13: $\#$ of latent confounders = 2, $\#$ of observed confounders = 1, random seed = 0}
\begin{figure}[H]
\centering
\includegraphics[scale=0.45]{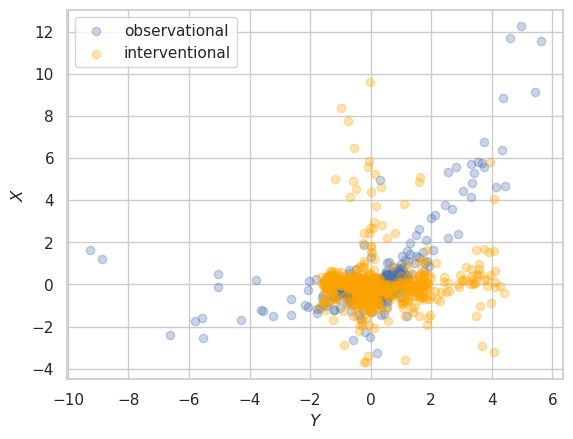}
\includegraphics[scale=0.45]{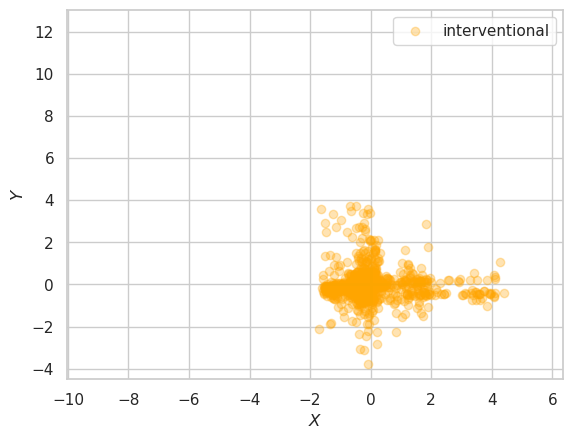}
\includegraphics[scale=0.45]{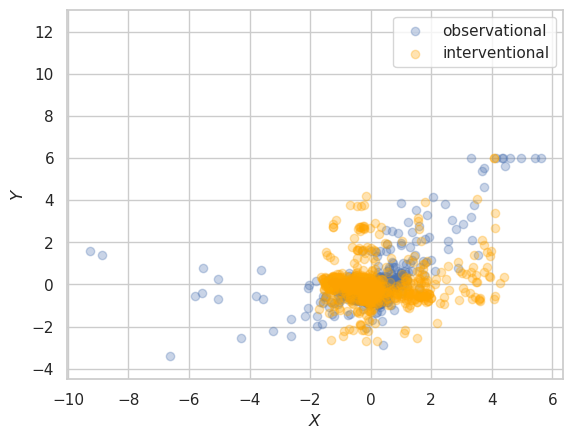}
\caption[Dataset 13: Interventional and observational samples.]{Dataset 13: Interventional and observational samples. (Top) training data. (Center) samples from flow model trained with only interventional data. (Bottom) samples from flow model trained with observational and interventional data.}
\end{figure}
\begin{figure}[H]
\centering
\includegraphics[scale=0.5]{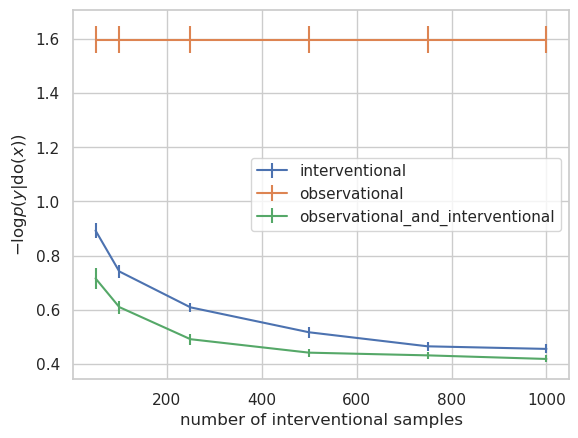}
\includegraphics[scale=0.5]{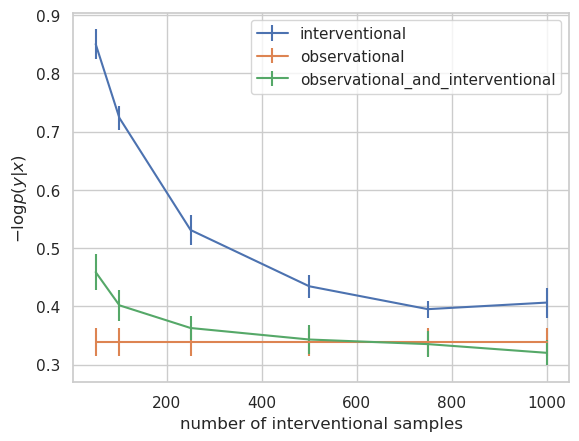}
\caption{Dataset 13: Performances measured in terms of negative log-likelihood on the interventional (top) and the observational (bottom) test sets.}
\end{figure}
\subsubsection{Dataset 14: $\#$ of latent confounders = 3, $\#$ of observed confounders = 3, random seed = 5}
\begin{figure}[H]
\centering
\includegraphics[scale=0.45]{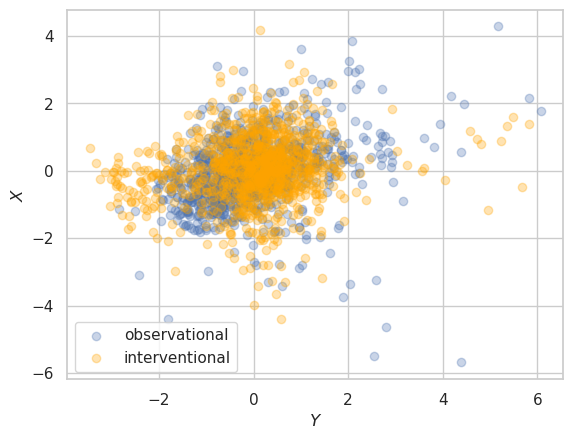}
\includegraphics[scale=0.45]{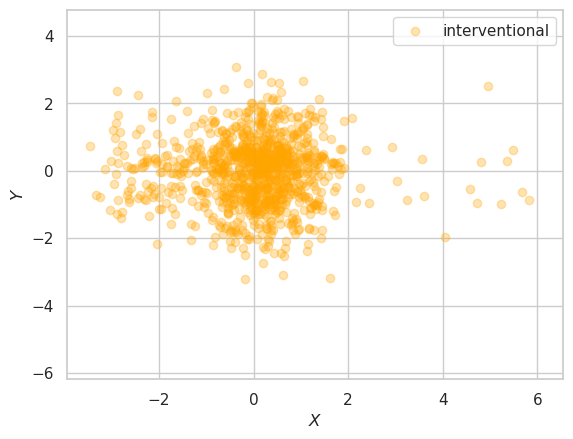}
\includegraphics[scale=0.45]{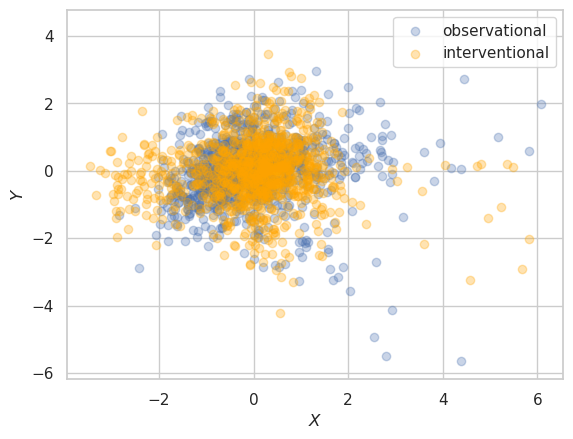}
\caption[Dataset 14: Interventional and observational samples.]{Dataset 14: Interventional and observational samples. (Top) training data. (Center) samples from flow model trained with only interventional data. (Bottom) samples from flow model trained with observational and interventional data.}
\end{figure}
\begin{figure}[H]
\centering
\includegraphics[scale=0.5]{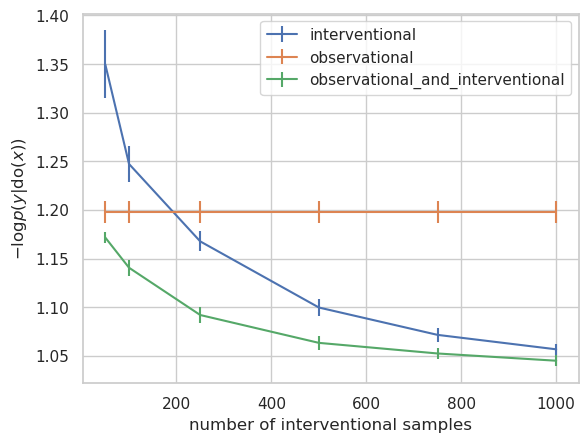}
\includegraphics[scale=0.5]{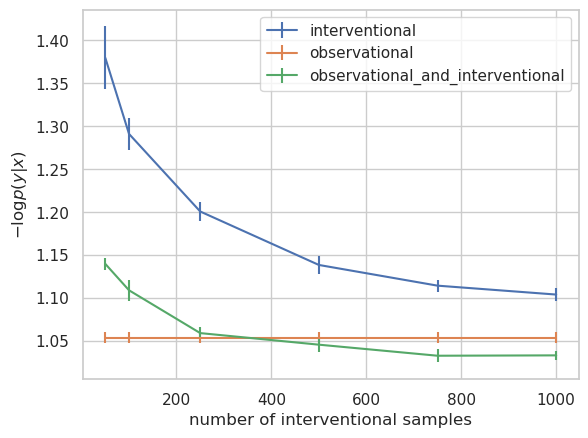}
\caption{Dataset 14: Performances measured in terms of negative log-likelihood on the interventional (top) and the observational (bottom) test sets.}
\end{figure}
\subsubsection{Dataset 15: $\#$ of latent confounders = 4, $\#$ of observed confounders = 4, random seed = 2}
\begin{figure}[H]
\centering
\includegraphics[scale=0.45]{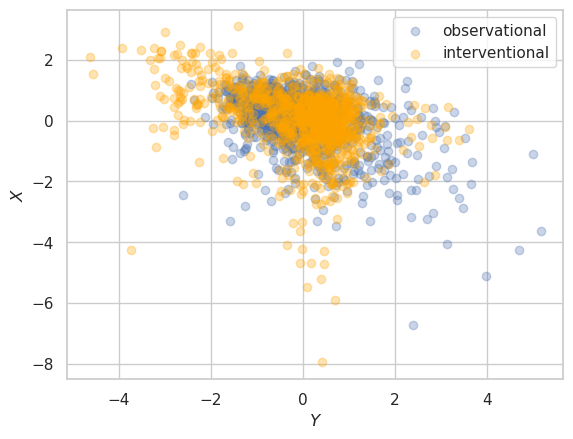}
\includegraphics[scale=0.45]{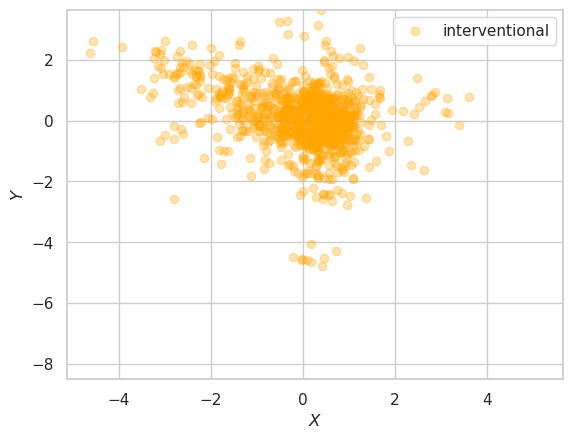}
\includegraphics[scale=0.45]{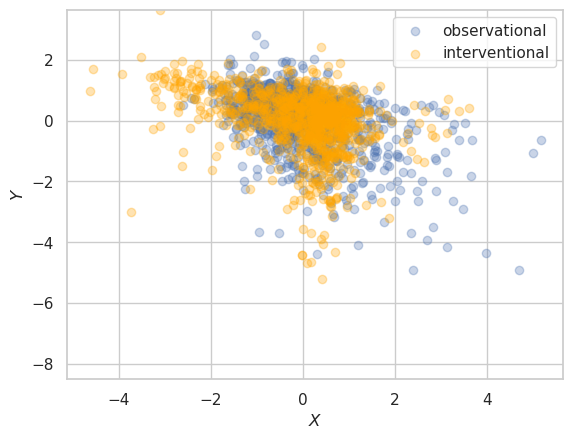}
\caption[Dataset 15: Interventional and observational samples.]{Dataset 15: Interventional and observational samples. (Top) training data. (Center) samples from flow model trained with only interventional data. (Bottom) samples from flow model trained with observational and interventional data.}
\end{figure}
\begin{figure}[H]
\centering
\includegraphics[scale=0.5]{nonlinear_data_observed_confounder/comparison_gp_seed_2_numZ_4_numC_4.png}
\includegraphics[scale=0.5]{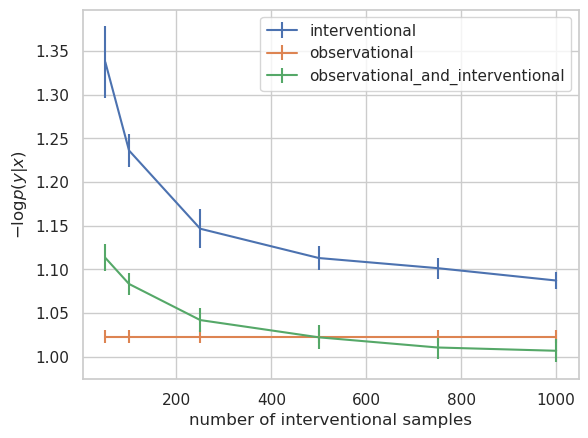}
\caption{Dataset 15: Performances measured in terms of negative log-likelihood on the interventional (top) and the observational (bottom) test sets.}
\end{figure}

\end{document}